\documentclass[twoside]{article}

\usepackage[preprint]{aistats2026}
\usepackage[round]{natbib}

\usepackage[utf8]{inputenc} % allow utf-8 input
\usepackage[T1]{fontenc}    % use 8-bit T1 fonts
\usepackage{hyperref}       % hyperlinks
\usepackage{url}            % simple URL typesetting
\usepackage{booktabs}       % professional-quality tables
\usepackage{amsfonts}       % blackboard math symbols
\usepackage{nicefrac}       % compact symbols for 1/2, etc.
\usepackage{microtype}      % microtypography
\usepackage{xcolor}         % colors

\usepackage{graphicx}
\usepackage{caption}
\usepackage{subfigure}
\usepackage{multirow}
\usepackage{bm}

\usepackage{amsmath}
\usepackage{amssymb}
\usepackage{mathtools}

\usepackage{enumitem}

\usepackage{algorithm}
\usepackage{algpseudocode}
\algnewcommand{\Input}{\item[\textbf{Input:}]}
\algnewcommand{\Output}{\item[\textbf{Output:}]}

\usepackage{array}

\newtheorem{theorem}{Theorem}[section]

\newtheorem{lemma}[theorem]{Lemma}
\newtheorem{corollary}[theorem]{Corollary}

\newtheorem{assumption}[theorem]{Assumption}

\newenvironment{proof}{{\noindent\it Proof.}}{\hfill $\square$\par}

\begin{document}

\runningtitle{On the Generalization Properties of Learning the RF Models with Learnable Activation Functions}


\twocolumn[

\aistatstitle{On the Generalization Properties of Learning the Random Feature Models with Learnable Activation Functions}

\aistatsauthor{ Zailin Ma \And Jiansheng Yang \And  Yaodong Yang }

\aistatsaddress{ Peking University \And  Peking University \And Peking University } ]

\begin{abstract}
  This paper studies the generalization properties of a recently proposed kernel method, the Random Feature models with Learnable Activation Functions (RFLAF). By applying a data-dependent sampling scheme for generating features, we provide by far the sharpest bounds on the required number of features for learning RFLAF in both the regression and classification tasks. We provide a unified theorem that describes the complexity of the feature number $s$, and discuss the results for the plain sampling scheme and the data-dependent leverage weighted scheme. Through weighted sampling, the bound on $s$ in the MSE loss case is improved from $\Omega(1/\epsilon^2)$ to $\tilde{\Omega}((1/\epsilon)^{1/t})$ in general $(t\geq 1)$, and even to $\Omega(1)$ when the Gram matrix has a finite rank. For the Lipschitz loss case, the bound is improved from $\Omega(1/\epsilon^2)$ to $\tilde{\Omega}((1/\epsilon^2)^{1/t})$. To learn the weighted RFLAF, we also propose an algorithm to find an approximate kernel and then apply the leverage weighted sampling. Empirical results show that the weighted RFLAF achieves the same performances with a significantly fewer number of features compared to the plainly sampled RFLAF, validating our theories and the effectiveness of this method.
\end{abstract}

\section{INTRODUCTION}

The kernel method is a power paradigm for learning complex functional patterns in machine learning. The empirical success of this method is due to the development of its scalable approximation, among which the random Fourier feature \citep{rahimi2008weighted} plays an important role. By constructing explicit low-dimensional features, Random feature (RF) models effectively reduce the computational and storage costs of learning. Despite its success, theoretical understandings of this model remain incomplete. Hence, extensive efforts are made to seek its statistical limits \citep{rudi2017generalization,avron2017random,bach2017equivalence,sun2018but}. Lately, \citet{li2021towards} provide a unified analysis and demonstrate that by applying a data-dependent sampling scheme, it is possible to achieve low excess risks with strikingly fewer features for plain RF models. Recently, \citet{ma2024random} proposed to study a new category of RF models, RFLAF, which incorporates the component of learnable activation functions in RF models, and demonstrated its advantages in various settings. From a theoretical perspective, they provide a basic estimate of the feature number of \(\Omega(1/\epsilon^2)\) within a data-independent sampling scheme. However, the estimate seems to be underdeveloped, and it remains unclear whether the techniques in \citep{li2021towards} can be transplanted to RFLAF. This leads us to the following question:

\emph{Is it possible to use a fewer number of features in RFLAF for faster computation while preserving the performance of low learning error?} 

This paper answers the question thoroughly by providing rigorous theoretical analyses and proposing an algorithm for empirical practice.  Our contributions are summarized as follows.

% \begin{itemize}[noitemsep]
%     \item We provide a unified theorem that estimates the required number of random features for any spectral measure from which features are sampled, and discuss the results for the plain sampling scheme and the leverage weighted sampling scheme specifically.
%     \item We derive the sharpest bounds on the feature number from the weighted sampling scheme in both the squared error loss case and the general Lipschitz loss case. For the squared error loss, our result also improves the order of the bound on sample size and the grid number of basis functions by one.
%     \item We propose an algorithm for applying leverage weighted sampling in learning RFLAF. We conduct numerical experiments to validate the effectiveness of the algorithm. Simulations show that the weighted models achieve the same performance with significantly fewer features than the plain models.
% \end{itemize}

\(\bullet\) We provide a unified theorem for RFLAF that estimates the required number of random features for any spectral measure from which features are sampled and discuss the results for the plain sampling scheme and the leverage weighted sampling scheme specifically.

\(\bullet\) We derive the sharpest bounds on the feature number from the weighted sampling scheme in both the squared error loss case and the general Lipschitz loss case. For the squared error loss, our result also improves the order of the bound on sample size and the grid number of basis functions by one.

\(\bullet\) We propose an algorithm to apply the leverage weighted sampling in learning RFLAF. We conduct numerical experiments to validate the effectiveness of the algorithm. Simulations show that the weighted models achieve the same performance with significantly fewer features than the plain models.

We discuss the related work in Appendix \ref{relatedwork}. In the rest of this paper, Section \ref{preliminary} contains the preliminaries and Section \ref{mainresults} presents the main theoretical results. In Section \ref{algorithm}, we propose an algorithm and then validate it in Section \ref{simulation}.

% \subsection{Related Work}

% \paragraph{Random Feature (RF) Models.}
% Random feature model \citep{rahimi2008weighted,rahimi2007random} is initially motivated by the fact that randomization is computationally cheaper than optimization \citep{amit1997shape,moosmann2006randomized}. Most recently, by virtue of the relations between a kernel and its Fourier spectral density, random features act as a technique 

\section{PRELIMINARIES}
\label{preliminary}
\subsection{Target Function Class}
In this paper, we focus on learning the infinite-width random feature model
\[
f(x) = \mathbb{E}_{w\sim \mu}[\sigma(w^\top x)v(w)],
\] with an arbitrary activation function \(\sigma\) supported on a compact set, that is, \(\sigma\in C_c(\mathcal{K})\). We assume a regularity condition that \(\sigma(\cdot)\) and \(v(\cdot)\) are Lipschitz continuous functions. Without loss of generality, we consider functions such that \(\mathbb{E}_{w\sim \mu}[v(w)^2]\leq 1\). Hence, \(\|f\|_{\mathcal{H}}^2\leq\mathbb{E}_{w\sim \mu}[v(w)^2]\leq 1\), where \(\mathcal{H}\) is the reproducing kernel Hilbert space (RKHS) induced by the spectral measure \(\mu\) and the activation function \(\sigma\). We denote the class of functions satisfying the above conditions as \(\mathcal{F}_{\mathcal{K}}\). We remark that the target function can also be seen as a two-layer neural network with the first-layer parameters initially sampled from the spectral measure \(\mu\) and then frozen during training. If \(\mu\) has a density function \(p(\cdot)\), then the kernel represented by the functions in the class can be written as
\begin{equation}
\label{kernelre}
\begin{split}
    k(x,y)&\,=\mathbb{E}_{w\sim p}\left[\sigma(w^\top x)\sigma(w^\top y)\right]\\
    &\,=\int_{\mathcal{W}}\sigma(w^\top x)\sigma(w^\top y)p(w)dw.
\end{split}
\end{equation}

\subsection{Random Feature Models with Learnable Activation Functions}

To learn the target functions, we consider the Random Feature models with Learnable Activation Functions (RFLAF) in \citep{ma2024random} formulated as
\begin{equation}
    \label{Prflaf}
	f(x;{a},{v}):=\sum_{m=1}^s\sum_{i=1}^N a_i B_i(w_m^\top x)v_m,
\end{equation}
where the random features \(\{w_m\}_{m=1}^s \overset{i.i.d.}{\sim} p(\cdot)\) are sampled in prior, and \({a}=(a_1,...,a_N)\in \mathbb{R}^{N}, {v}=(v_1,...,v_s)\in\mathbb{R}^s\) are learnable parameters.

In this work, we consider radial basis functions (RBF) \(B_i(x):=\exp(-(x-c_i)/2h^2)\) as a typical category of basis functions for all theoretical results, in which \(\{c_i\}_{i\in[N]}\) are the grid points evenly distributed on the compact set \(\mathcal{K}\) and \(h\) is the width of the basis functions. The component of learnable activation function is \(\tilde{\sigma}(x):=\sum_{i=1}^{N}\tilde{a}_iB_i(x)\). If we denote \(W=[w_1, w_2,..., w_s]\in \mathbb{R}^{d\times s}\), then we have a compact expression as \(f(x;{a},{v})=\sum_{i=1}^N a_i B_i(x^\top W)v\). Intuitively, for any target function \(f^*(x)\in \mathcal{F}_{\mathcal{K}}\), we wish the approximate activation function \(\tilde{\sigma}(x)\) and \(\tilde{f}(x) := \mathbb{E}_{w}\left[\tilde{\sigma}(w^\top x)v(w)\right]\) to be close to the ground truth \(\sigma^*(x)\) and \(f^*(x)\) respectively. Proposition 4.1 in \citep{ma2024random} provides an estimate of the approximation error:
\begin{lemma}[Proposition 4.1 in \citep{ma2024random}]
    \label{lemma21}
    For any \(\epsilon>0\), if we set the width \(h\) and the grid number \(N\) as 
    \(\frac{1}{h}=\Omega\left(\frac{1}{\epsilon}\left(\log{\frac{1}{\epsilon}}\right)^{\frac{1}{2}}\right)\), \(N=\Omega\left(\frac{1}{\epsilon^2}{\left(\log{\frac{1}{\epsilon}}\right)^\frac{3}{2}}\right)\),
    then there exists \(\{\tilde{a}_i\}_{i=1}^N\) such that the approximation errors are bounded as
    \begin{equation}
        \label{tilde}
        \|\tilde{\sigma}(x)-\sigma^*(x)\|_{\infty}\leq \epsilon,\quad \|\tilde{f}(x)-f^*(x)\|_{\infty}\leq \epsilon.
    \end{equation}
\end{lemma}

\subsection{Ridge Leverage Score Function}

The ridge leverage score function is an important component for efficient sampling in our algorithm. It determines the weighted sampling distribution that allows for fewer random feature numbers to achieve low test errors. Suppose that we are given $n$ samples of data \(X=[x_1,...,x_n]^\top\in\mathbb{R}^{n\times d}\). The Gram matrix induced by the approximate activation function \(\tilde{\sigma}\) in Lemma \ref{lemma21} is 
\(
    \widetilde{\mathbf{K}} := \mathbb{E}_{w\sim p}\left[\tilde{\sigma}(Xw)\tilde{\sigma}(Xw)^\top\right].
\)

Given a regularization parameter \(\lambda\), the ridge leverage score function \citep{li2021towards} with respect to the approximate kernel and data samples is
\[
    l_{\lambda}(w) := p(w)\tilde{\sigma}(Xw)^\top( \widetilde{\mathbf{K}}+n\lambda \mathbf{I})^{-1}\tilde{\sigma}(Xw),
\]
and the effective dimension (or the number of effective degrees of freedom) is defined as
\[
    d_{\widetilde{\mathbf{K}}}^{\lambda}:=\mathrm{Tr}[\widetilde{\mathbf{K}}(\widetilde{\mathbf{K}}+n\lambda \mathbf{I})^{-1}]=\int_{\mathcal{W}}l_{\lambda}(w)dw.
\]
For the ridge leverage score function, we have a trivial upper bound
\(
    l_{\lambda}(w)\leq p(w)\frac{\|\sigma\|_{\infty}^2}{\lambda},
\)
which is due to the fact that \(( \widetilde{\mathbf{K}}+n\lambda \mathbf{I})^{-1}\preceq \left( n\lambda \mathbf{I}\right)^{-1}\).

We make the following assumption on the regularization parameter \(\lambda\) as in \citep{li2021towards} that holds throughout our analysis. Intuitively, the assumption requires that the strongest signal \(\lambda_1\) is stronger than the added regularization term \(\lambda\) to avoid underfitting.
\begin{assumption}
    \label{assump22}
    For kernel \(\tilde{k}\), denote with \(\lambda_1\geq \cdots \geq \lambda_n\) the eigenvalues of the normalized Gram matrix \(\widetilde{\mathbf{K}}/n\). We assume that the regularization parameter satisfies \(0\leq \lambda\leq \lambda_1\).
\end{assumption}

\subsection{Weighted Sampling for Generating Random Features}

Suppose that we want to apply weighted sampling to generate random features. Let the importance weighted density function of the random feature vectors be \(q(w)\), and \(\{w_m\}_{m=1}^s \overset{i.i.d.}{\sim} q(\cdot)\). Denote the weight matrix by \(Q=\mathrm{diag}\{\sqrt{p(w_1)/q(w_1)},...,\sqrt{p(w_s)/q(w_s)}\}\in\mathbb{R}^{s\times s}\). Then the weighted Random Feature models with Learnable Activation Functions have the form of 
\begin{equation}
    \label{Wrflaf}
    f(x;{a},{v})=\sum_{i=1}^N a_i B_i(x^\top W)Qv.
\end{equation}

Let \(\sigma_a(w^\top x) = \sum_{k=1}^N a_kB_k(w^\top x)\), and \(z_{q,a}(w, x) = \sqrt{p(w)/q(w)}\sigma_a(w^\top x)\). Define the in-sample nonlinear feature matrix \(\mathbf{Z}_q({a}):=\sigma_a(XW)Q \in \mathbb{R}^{n\times s}\) such that \(z_{q,a}(w_j, x_i)\) is the \((i,j)\)-th entry of \(\mathbf{Z}_q({a})\). Then we denote the Gram matrices as 
\[
    \widehat{\mathbf{K}}_q(a):=\frac{1}{s}\mathbf{Z}_q({a})\mathbf{Z}_q({a})^\top,
    \quad
    \mathbf{K}_q(a) := \mathbb{E}_{w\sim q}[\widehat{\mathbf{K}}_q(a)].
\]
Specifically, we have that \(\widetilde{\mathbf{K}}=\mathbf{K}_q(\tilde{a})\) for the approximate activation function.

% The parameter \(\tilde{a}\) provides an approximated nonlinear mapping \(\tilde{\sigma}\) to the ground truth nonlinearity. This result applies for both plain and weighted RFLAFs.

% !!! Modify the remain words.

% For plain RFLAF specifically, let $\{w_m\}_{m=1}^{s}\overset{i.i.d}{\sim} \mathcal{N}(0,1)$, then with high probability, there exists \(\boldsymbol{a}\) and \(\boldsymbol{v}\) such that
% \[
%     \mathbb{E}_{x}\left|\hat{f}(x;{a},{v})-f^*(x)\right|\leq O\left(\frac{1}{\sqrt{s}}\right)+\epsilon,
% \]
% and
% \begin{equation}
%     \label{avbound}
%     \|{a}\|_2^2 =  O\left(\frac{1}{h^2N}\right)=O\left(\left({\log{\frac{1}{\epsilon}}}\right)^{-\frac{1}{2}}\right),\quad \|{v}\|_2^2 = O\left(\frac{1}{s}\right).
% \end{equation}

% The two results state that plain RFLAF can approximate the target function to an arbitrarily low error with sufficently large number of grids \(N\) and random features \(s\). For the second result, we will provide an improved result on the required random feature number in this work.

\section{THEORETICAL RESULTS}
\label{mainresults}

% In this section, we formally present the main results of this paper. 

\subsection{Learning with the Squared Error Loss}

The mean squared error loss \(l(f(x),y)=(f(x)-y)^2\) is the standard loss function in regression tasks. For the weighted RFLAF, we aim to analyze the optimization problem of the regression task formulated as follows.
\begin{equation}
    \label{mainpro}
    \min_{\|a\|_2\leq R}\left\{\min_{v}\frac{1}{n} \left\|\mathbf{Z}_q({a})v-y\right\|^2+\lambda s \|v\|_2^2\right\},
\end{equation}
where the regularization parameter \(\lambda = o(1)\) as \(n\rightarrow \infty\) (e.g., \(\lambda = 1/\sqrt{n}\)). In this section, we focus on the following question:

\emph{What is the least required number of random features for RFLAF to approximate the minimizer to a sufficiently low error?} 

Specifically, we will provide a sharper upper bound on the number of random features required to achieve \(O(1/\sqrt{n})\) error with respect to the worst case of the regression problem. Denote the minimizer of the problem (\ref{mainpro}) as
\begin{equation}
    \label{mzer}
    \hat{f}(x) = \mathop{\mathrm{argmin}}_{f(x)\in \mathcal{F}_V}\frac{1}{n}\sum_{i=1}^{n} \left(f(x_i)-y_i\right)^2+\lambda s \|v\|_2^2,
\end{equation}
where the hypothesis class is
\[
\mathcal{F}_V:=\left\{f(x)=\sum_{k=1}^{N}a_kB_k(x^\top W)Qv:\|a\|_2\leq R,v\in\mathbb{R}^s\right\}
\]
and \(R = O(1/h\sqrt{N})\) which is indicated by Theorem 4.2 in \citep{ma2024random}. The optimal parameters of \(\hat{f}(x)\) are denoted by \(\hat{a}\) and \(\hat{v}\). Then we make the following assumption on the data.

\begin{assumption}
    \label{assump1}
    We assume that \(y = f^*(x)+\varepsilon\) where \(f^*\in\mathcal{F}_{\mathcal{K}}\) and \(\varepsilon\) is a Gaussian random variable such that \(\mathbb{E}[\varepsilon]=0\) and \(\mathrm{Var}[\varepsilon]=\sigma^2\). In addition, we assume that the sampled data \(y\) are bounded such that \(|y|\leq |y_0|\).
\end{assumption}
We denote the empirical learning risk and population learning risk by
\[\mathbb{E}_n[l_f]:=\frac{1}{n}\sum_{i=1}l(f(x_i),y_i),~~\mathbb{E}[l_f]:=\mathbb{E}_{(x,y)\sim \mathtt{P}}[l(f(x),y)].\]
We formally present the main result that answers the aforementioned question. The proof of this theorem is contained in Appendix \ref{proof31}. 
\begin{theorem}
    \label{thm31}
    Under the Assumption \ref{assump22}, \ref{assump1}, for any \(\epsilon>0\), we set the width \(h\) and the grid number \(N\) as 
    \(\frac{1}{h}=\Omega\left(\left(\frac{1}{\epsilon}\log{\frac{1}{\epsilon}}\right)^{\frac{1}{2}}\right)\), \(N=\Omega\left(\frac{1}{\epsilon}{\left(\log{\frac{1}{\epsilon}}\right)^\frac{3}{2}}\right)\). Let \(\tilde{l}:\mathbb{R}^s \rightarrow \mathbb{R}\) be a measurable function such that for all \(w\in \mathbb{R}^s\), \( \tilde{l}(w) \geq l_{\lambda}(w)\) with \(d_{\tilde{l}}=\int_{\mathbb{R}^s} \tilde{l}(w)dv < \infty\). Suppose that \(\{w_i\}_{i=1}^s\) are sampled independently from the probability density function \(q(w)=\tilde{l}(w)/d_{\tilde{l}}\). For all \(\delta \in (0,1)\), if
    \[
        s \geq 5d_{\tilde{l}}\log{\frac{16d_{\widetilde{\mathbf{K}}}^{\lambda}}{\delta}},
    \]
    then with probability \(1-\delta\), the excess risk of \(\hat{f}\) can be upper bounded by
    \[
        \mathbb{E}[l_{\hat{f}}]-\mathbb{E}[l_{f^*}]\leq 4\lambda + O\left(\frac{1}{h\sqrt{n}}\right) + 2\epsilon.
    \]
\end{theorem}
% We apply the analysis framework in \citep{li2021towards} and extend the results to the case of random feature models with learnable activation functions.

The significance of this theorem is three-fold. It provides sharper bounds on three critical parameters: the sample size \(n\), the grid number \(N\) and the number of random features \(s\) in the squared error loss case compared to previous results in \citep{ma2024random}. The first direct implication of Theorem \ref{thm31} is that it improves the required sample size \(n\) from \(\widetilde{\Omega}\left(1/\epsilon^4\right)\) to \(\widetilde{\Omega}\left(1/\epsilon^3\right)\) and improves the required grid number \(N\) from \(\widetilde{\Omega}\left(1/\epsilon^2\right)\) to \(\widetilde{\Omega}\left(1/\epsilon\right)\), indicating that a smaller sample size or grid number is sufficient to achieve \(\epsilon\) learning error.

\begin{table*}[tbp]
\renewcommand{\arraystretch}{1.8}
\caption{Summarization of the Number of Random Features in the MSE Loss Case.}
\label{summse}
\begin{center}
\begin{small}
\begin{sc}
\begin{tabular}{l|c|c|c}
\noalign{\vskip 0.5pt}      
\noalign{\hrule height 0.75pt}
\noalign{\vskip 0.5pt}
sampling scheme & spectrum & number of features & excess risk \\ 
\hline
\multirow{3}{*}{weighted rflaf} 
& \rm{finite rank} & $s \in \Omega(1)$ & \multirow{3}{*}{\(O\left(\epsilon\right)\)} \\ \cline{2-3}
& $\lambda_i \propto A^i$ & $s \in \Omega(\log{(1/\epsilon)}\cdot\log{\log{(1/\epsilon)}})$ & \\ \cline{2-3}
& $\lambda_i \propto i^{-t}$ ($t \geq 1$) & $s \in \Omega((1/\epsilon)^{1/t} \cdot \log{(1/\epsilon)})$ & \\ %\cline{2-3}
% & $\lambda_i \propto i^{-1}$ & $s \in \Omega(\sqrt{n} \cdot \log n)$ & \\ 
\hline
\multirow{3}{*}{plain rflaf} 
& \rm{finite rank} & $s \in \Omega(1/\epsilon)$ & \multirow{3}{*}{\(O\left(\epsilon\right)\)} \\ \cline{2-3}
& $\lambda_i \propto A^i$ & $s \in \Omega({(1/\epsilon)}\cdot\log{\log{(1/\epsilon)}})$ & \\ \cline{2-3}
& $\lambda_i \propto i^{-t}$ ($t \geq 1$) & $s \in \Omega((1/\epsilon)\cdot\log{(1/\epsilon)})$ & \\ %\cline{2-3}
% & $\lambda_i \propto i^{-1}$ & $s \in \Omega()$ & \\ 
\hline
former result \citep{ma2024random} & \rm{in any case} & \(s\in \Omega(1/\epsilon^2)\) & \(O\left(\epsilon\right)\) \\
\noalign{\vskip 0.5pt}      
\noalign{\hrule height 0.75pt}
\noalign{\vskip 0.5pt}
\end{tabular}
\end{sc}
\end{small}
\end{center}
\end{table*}

For the required number of random features \(s\), the result shows that the bound for \(s\) in plain RF models derived in \citep{li2021towards} also holds in RFLAF. Specifically, we consider two sampling strategies: \emph{leverage weighted sampling} (Corollary \ref{cor1}) and \emph{plain sampling} (Corollary \ref{cor2}), and discuss the corresponding number of random features.

\begin{corollary}[Leverage Weighted Sampling]
    \label{cor1}
    If setting the probability density function in Theorem \ref{thm31} as the ridge leverage score distribution
    \(
        q(w) = l_{\lambda}(w)/d_{\widetilde{\mathbf{K}}}^{\lambda},
    \)
    then the upper bound on the excess risk holds for all \[s\geq 5d_{\widetilde{\mathbf{K}}}^{\lambda}\log{({16d_{\widetilde{\mathbf{K}}}^{\lambda}}/{\delta})}.\]

\end{corollary}

Corollary \ref{cor1} indicates that if we sample features according to the ridge leverage score function, the required number is \(s=\Omega(d_{\widetilde{\mathbf{K}}}\log{d_{\widetilde{\mathbf{K}}}})\). To achieve \(O(\epsilon)\) excess risk with sample complexity \(n=\Omega\left(1/\epsilon^2h^2\right)\), it suffices to have \(\lambda=O(\epsilon)\), which eventually determines the bounds on \(d_{\widetilde{\mathbf{K}}}^{\lambda}\) and \(s\). Detailed derivations of the upper bound of \(d_{\widetilde{\mathbf{K}}}^{\lambda}\) are provided in Appendix \ref{effdim}, which is estimated by \(\lambda\) and the eigenvalues of \(\widetilde{\mathbf{K}}/n\). Below, we provide results in four cases of the effective dimension. In the best case, if the number of positive eigenvalues is finite, then \(d_{\widetilde{\mathbf{K}}}^{\lambda}=O(1)\) is a constant that does not grow with \(n\). Hence, we simply have \(s=\Omega(1)\) which implies that only a finite number of random features is sufficient. Next, if the eigenvalues decay exponentially, that is, \(\lambda_i \propto A^i\), then we have \(d_{\widetilde{\mathbf{K}}}^{\lambda}\leq O(\log{(1/\lambda)})\), leading to the result of \(s=\Omega(\log{({1}/{\epsilon})}\log{\log{({1}/{\epsilon})}})\). In the case of slow decay with \(\lambda_i \propto i^{-t}~ (t> 1)\), we have \(d_{\widetilde{\mathbf{K}}}^{\lambda}\leq O((1/\lambda)^{1/t})\), and therefore \(s=\Omega(({1}/{\epsilon})^{{1}/{t}}\log{({1}/{\epsilon})})\). In the worst case that \(\lambda_i \propto i^{-1}\), we have \(d_{\widetilde{\mathbf{K}}}^{\lambda}\leq O(1/\lambda\cdot \log{n})\). Since \(n=\Omega(1/\epsilon^2h^2)\), we obtain \(s=\Omega(({1}/{\epsilon})\log{({1}/{\epsilon})})\).

\begin{corollary}[Plain Sampling]
    \label{cor2}
    If setting the density function in Theorem \ref{thm31} as the spectral measure
    \( 
        q(w) = p(w),
    \)
    then the upper bound on the excess risk holds for all \[s\geq \frac{5{\|\tilde{\sigma}\|_{\infty}^2}}{\lambda}\log{\frac{16d_{\widetilde{\mathbf{K}}}^{\lambda}}{\delta}}.\]
\end{corollary}
In the plain sampling scheme, we sample features according to the spectral measure \(p(w)\). To achieve \(O(\epsilon)\) excess risk, let \(\lambda=O(\epsilon)\), the required number of random features is \(s=\Omega(1/\epsilon\cdot\log{d_{\widetilde{\mathbf{K}}}^\lambda})\). Similarly, we have four cases of the effective dimension \(d_{\widetilde{\mathbf{K}}}^{\lambda}\). In the best case that the number of positive eigenvalues is finite, we have \(d_{\widetilde{\mathbf{K}}}^{\lambda}=O(1)\) and hence \(s=\Omega(1/\epsilon)\). If \(\lambda_i \propto A^i\), then \(d_{\widetilde{\mathbf{K}}}^{\lambda}\leq O(\log{(1/\lambda)})\) and \(s=\Omega(({1}/{\epsilon})\log\log{({1}/{\epsilon})})\). If \(\lambda_i \propto i^{-t}~ (t> 1)\), then \(d_{\widetilde{\mathbf{K}}}^{\lambda}\leq O((1/\lambda)^{1/t})\) and \(s=\Omega(({1}/{\epsilon})\log{({1}/{\epsilon})})\). In the worst case that \(\lambda_i \propto i^{-1}\), we have \(d_{\widetilde{\mathbf{K}}}^{\lambda}\leq O(1/\lambda\cdot \log{n})\) and \(s=\Omega(({1}/{\epsilon})\log{({1}/{\epsilon})})\). Table \ref{summse} summarizes the above discussions and provides comparisons with previous results for the MSE loss case.

% \paragraph{Discussion}

% Table of summarization here.

% We meant to say that there exists a distribution of \(v\) (probably unknown) such that it is possible to sample only a very few number of random features to achieve a considerably low excess risk. However, how to sample or find a proper distribution in algorithmic practice is another problem to solve. In section 4 we propose a procedure to find the approximated nonlinearity first, and then apply the weighted sampling.

\subsection{Learning with a Lipschitz Continuous Loss}

\begin{table*}[tbp]
\renewcommand{\arraystretch}{1.8}
\caption{Summarization of the Number of Random Features in the Lipschitz Continuous Loss Case.}
\label{sumlip}
\begin{center}
\begin{small}
\begin{sc}
\begin{tabular}{l|c|c|c}
\noalign{\vskip 0.5pt}      
\noalign{\hrule height 0.75pt}
\noalign{\vskip 0.5pt}
sampling scheme & spectrum & number of features & excess risk \\ 
\hline
\multirow{3}{*}{weighted rflaf} 
& \rm{finite rank} & $s \in \Omega(1)$ & \multirow{3}{*}{\(O\left(\epsilon\right)\)} \\ \cline{2-3}
& $\lambda_i \propto A^i$ & $s \in \Omega(\log{(1/\epsilon)}\cdot\log{\log{(1/\epsilon)}})$ & \\ \cline{2-3}
& $\lambda_i \propto i^{-t}$ ($t \geq 1$) & $s \in \Omega((1/\epsilon^2)^{1/t} \cdot \log{(1/\epsilon)})$ & \\ 
\hline
\multirow{3}{*}{plain rflaf} 
& \rm{finite rank} & $s \in \Omega(1/\epsilon^2)$ & \multirow{3}{*}{\(O\left(\epsilon\right)\)} \\ \cline{2-3}
& $\lambda_i \propto A^i$ & $s \in \Omega({(1/\epsilon^2)}\cdot\log{\log{(1/\epsilon)}})$ & \\ \cline{2-3}
& $\lambda_i \propto i^{-t}$ ($t \geq 1$) & $s \in \Omega((1/\epsilon^2)\cdot\log{(1/\epsilon)})$ & \\ 
\hline
former result \citep{ma2024random} & \rm{in any case} & \(s\in \Omega(1/\epsilon^2)\) & \(O\left(\epsilon\right)\) \\
\noalign{\vskip 0.5pt}      
\noalign{\hrule height 0.75pt}
\noalign{\vskip 0.5pt}
\end{tabular}
\end{sc}
\end{small}
\end{center}
\end{table*}

In this section, we consider a more general case where the loss function is Lipschitz continuous with respect to the first variable, examples of which include kernel logistic or softmax regression in classification tasks. The general optimization problem is formulated as follows.
\begin{equation}
    \label{mzer2}
    \hat{g}(x) = \mathop{\mathrm{argmin}}_{g(x)\in\mathcal{F}_V}\frac{1}{n}\sum_{i=1}^{n} l\left(g(x_i),y_i\right)+\lambda s \|v\|_2^2.
\end{equation}

Specifically, we have the Lipschitz assumption on the loss function.
\begin{assumption}
    \label{lipsassump}
    Assume the loss function \(l(\cdot,\cdot)\) is Lipschitz continuous with respect to the first variable with Lipschitz constant \(L=1\), which means that \(\forall g,g^\prime\in\mathcal{H},\forall x\in\mathcal{X}\), it holds that \[  
        |l(g(x),y)-l(g^\prime(x),y)|\leq L|g(x)-g^\prime(x)|.
    \]
\end{assumption}

In the theorem, we provide a sharp bound that describes the relations between the learning risk and the random feature number \(s\).
\begin{theorem}
    \label{thm36}
    Under the Assumption \ref{assump22}, \ref{lipsassump}, for any \(\epsilon>0\), we set the width \(h\) and the grid number \(N\) as 
    \(\frac{1}{h}=\Omega\left(\frac{1}{\epsilon}\left(\log{\frac{1}{\epsilon}}\right)^{\frac{1}{2}}\right)\), \(N=\Omega\left(\frac{1}{\epsilon^2}{\left(\log{\frac{1}{\epsilon}}\right)^\frac{3}{2}}\right)\). Let \(\tilde{l}:\mathbb{R}^s \rightarrow \mathbb{R}\) be a measurable function such that for all \(w\in \mathbb{R}^s\), \( \tilde{l}(w) \geq l_{\lambda}(w)\) with \(d_{\tilde{l}}=\int_{\mathbb{R}^s} \tilde{l}(w)dv < \infty\). Suppose that \(\{w_i\}_{i=1}^s\) are sampled independently from the probability density function \(q(w)=\tilde{l}(w)/d_{\tilde{l}}\). For all \(\delta \in (0,1)\), if
    \[
        s \geq 5d_{\tilde{l}}\log{\frac{16d_{\widetilde{\mathbf{K}}}^{\lambda}}{\delta}},
    \]
    then with probability \(1-\delta\), the excess risk of \(\hat{g}\) can be upper bounded by
    \[
        \mathbb{E}[l_{\hat{g}}]-\mathbb{E}[l_{g^*}]\leq 4\sqrt{\lambda} + O\left(\frac{1}{h\sqrt{n}}\right) + \epsilon.
    \]
\end{theorem}

The proof of this theorem is contained in Appendix \ref{proof36}. For Lipschitz losses, the theorem contributes to an improvement only for \(s\) within the leverage weighted sampling scheme. The corollaries are parallel to the case of MSE loss. We present Corollary \ref{cor7} for \emph{leverage weighted sampling} and Corollary \ref{cor8} for \emph{plain sampling}.
\begin{corollary}[Leverage Weighted Sampling]
    \label{cor7}
    If setting the probability density function in Theorem \ref{thm36} as the ridge leverage score distribution
    \(
        q(w) = l_{\lambda}(w)/d_{\widetilde{\mathbf{K}}}^{\lambda},
    \)
    then the upper bound on the excess risk holds for all \[s\geq 5d_{\widetilde{\mathbf{K}}}^{\lambda}\log{({16d_{\widetilde{\mathbf{K}}}^{\lambda}}/{\delta})}.\]
\end{corollary}
In the leverage weighted sampling scheme, to achieve \(O(\epsilon)\) excess risk with sample complexity \(n=\Omega\left(1/\epsilon^2h^2\right)\), it suffices to have \(\lambda=O(\epsilon^2)\). We consider four cases of \(d_{\widetilde{\mathbf{K}}}^{\lambda}\). In the best case, \(\widetilde{\mathbf{K}}\) has finite rank, then \(d_{\widetilde{\mathbf{K}}}^{\lambda}=O(1)\). Hence, we have \(s=\Omega(1)\). If the eigenvalues decay exponentially, that is, \(\lambda_i \propto A^i\), then we have \(d_{\widetilde{\mathbf{K}}}^{\lambda}\leq O(\log{(1/\lambda)})\), leading to the result of \(s=\Omega(\log{({1}/{\epsilon})}\log{\log{({1}/{\epsilon})}})\). In the case of slow decay with \(\lambda_i \propto i^{-t}~ (t> 1)\), we have \(d_{\widetilde{\mathbf{K}}}^{\lambda}\leq O((1/\lambda)^{1/t})\), and therefore \(s=\Omega(({1}/{\epsilon^2})^{{1}/{t}}\log{({1}/{\epsilon})})\). In the worst case that \(\lambda_i \propto i^{-1}\), we have \(d_{\widetilde{\mathbf{K}}}^{\lambda}\leq O(1/\lambda\cdot \log{n})\). Since \(n=\Omega(1/\epsilon^2h^2)\), we obtain \(s=\Omega(({1}/{\epsilon^2})\log{({1}/{\epsilon})})\).

\begin{algorithm*}[tbp]
\caption{Learning weighted RFLAF with Approximate Leverage Weighted Sampling}
\label{alg1}
\begin{algorithmic}[1]
\Input sample of examples $\{(x_i, y_i)\}_{i=1}^n$, regularization parameters $\lambda_0$, $\lambda$, $\lambda^*$, the spectral density function \(p(w)\) and the number of random feature pool \(s\) and the number of weighted features \(S\).
\Output feature matrix \(\widetilde{W}\), weight matrix \(Q\) and the model parameters \(a, v\).
\State sample a pool of $s$ random features $\{w_1, \ldots, w_s\}$ from $p(w)$ to form the initial feature matrix \(W\)
\State solve the optimization problem to obtain $\tilde{a}$ and $\tilde{v}$:
\begin{equation}
  \label{matrixsensing}
  \tilde{a}, \tilde{v} =\mathop{\mathrm{argmin}}_{a, v} \frac{1}{n} \|\mathbf{Z}_p(a) v - y\|_2^2 + \lambda_0 (\|a\|_2^2 - \|v\|_2^2)^2
\end{equation}
\State associate with each feature $w_i$ a positive real number $p_i$ such that $p_i$ is equal to the $i$-th diagonal element of matrix
\[
\mathbf{Z}_p(\tilde{a})^T \mathbf{Z}_p(\tilde{a}) \left( (1/s) \mathbf{Z}_p(\tilde{a})^T \mathbf{Z}_p(\tilde{a}) + n \lambda I \right)^{-1}
\]
\State $q(w_i) \gets p_i/\sum_{i=1}^s p_i$ and $M \gets \{(w_i, q(w_i))\}_{i=1}^s$
\State resample features \(\widetilde{W}=[w_1,...,w_S]\) from set $M$ using the multinomial distribution given by vector $(q(w_1), \cdots, q(w_s))$, and record the corresponding weights \(Q=\mathrm{diag}\{\sqrt{1/q(w_1)},...,\sqrt{1/q(w_S)}\}\)
\State solve the optimization problem (\ref{mzer}) with \(\lambda^*\), \(\widetilde{W}\) and \(Q\), and return the optimized \(a, v\)
\end{algorithmic}
\end{algorithm*}

\begin{corollary}[Plain Sampling]
    \label{cor8}
    If setting the density function in Theorem \ref{thm36} as the spectral measure
    \( 
        q(w) = p(w),
    \)
    then the upper bound on the excess risk holds for all \[s\geq \frac{5{\|\tilde{\sigma}\|_{\infty}^2}}{\lambda}\log{\frac{16d_{\widetilde{\mathbf{K}}}^{\lambda}}{\delta}}.\]
\end{corollary}
In the plain sampling scheme, to achieve \(O(\epsilon)\) excess risk, let \(\lambda=O(\epsilon^2)\), then the required number of random features is \(s=\Omega(1/\epsilon^2\cdot\log{d_{\widetilde{\mathbf{K}}}^\lambda})\). Similarly, we have four cases of \(d_{\widetilde{\mathbf{K}}}^{\lambda}\). In the best case that the number of positive eigenvalues is finite, we have \(d_{\widetilde{\mathbf{K}}}^{\lambda}=O(1)\) and hence \(s=\Omega(1/\epsilon^2)\). If \(\lambda_i \propto A^i\), then \(d_{\widetilde{\mathbf{K}}}^{\lambda}\leq O(\log{(1/\lambda)})\) and \(s=\Omega({({1}/{\epsilon^2})}\log{\log{({1}/{\epsilon})}})\). If \(\lambda_i \propto i^{-t}~ (t> 1)\), then \(d_{\widetilde{\mathbf{K}}}^{\lambda}\leq O((1/\lambda)^{1/t})\) and \(s=\Omega(({1}/{\epsilon^2})\log{({1}/{\epsilon})})\). In the worst case that \(\lambda_i \propto i^{-1}\), we have \(d_{\widetilde{\mathbf{K}}}^{\lambda}\leq O(1/\lambda\cdot \log{n})\) and \(s=\Omega(({1}/{\epsilon^2})\log{({1}/{\epsilon})})\). Table \ref{sumlip} summarizes the above discussions and provides comparisons with previous results for the Lipschitz loss case.

\subsection{Discussions of the Main Results}

% In the remainder of this section, we will discuss our results relative to previous results in \citep{li2021towards} and \citep{ma2024random}.

\paragraph{Comparison with previous results}
% The model RFLAF is initally proposed by \citet{ma2024random} and the author provides a basic worst-case analysis regarding the sample number \(n\), the grid number \(N\) and the random feature number \(s\). 

The theoretical improvements can be summarized into four aspects. First, in the case of squared error loss with plain sampling scheme, the upper bound on \(s\) is lowered from \(\Omega(1/\epsilon^2)\) to \(\widetilde{\Omega}(1/\epsilon)\). Second, in the case of squared error loss with leverage weighted sampling scheme, the bound is further improved to \(\widetilde{\Omega}(1/\epsilon^{1/t}),~t\geq 1\) (Table \ref{summse}). Third, in the most general case where losses are Lipschitz continuous, by applying the leverage weighted sampling scheme, we obtain a sharper bound on the number of random features \(s\) from \(\Omega(1/\epsilon^2)\) to \(\Omega(1/\epsilon^{2/t}),~t\geq 1\) (Table \ref{sumlip}). To conclude, except from the case of Lipschitz continuous losses with plain sampling scheme, the estimated order of \(1/\epsilon\) in all other cases decreases. For both kinds of losses, if using the leverage weighted sampling scheme, the complexity becomes \(\widetilde{\Omega}(\log(1/\epsilon))\) if the eigenvalues of the Gram matrix decay exponentially, and is further reduced to \(\Omega(1)\) if the Gram matrix has finite rank. Both quantities are considerably low in practice, indicating that even a constant number of features is sufficient for low learning errors. The last improvement to highlight is that in the case of squared error loss, Theorem \ref{thm31} also lowers the bound of the required sample size \(n\) from \(\widetilde{\Omega}\left(1/\epsilon^4\right)\) to \(\widetilde{\Omega}\left(1/\epsilon^3\right)\) and improves the grid number \(N\) from \(\widetilde{\Omega}\left(1/\epsilon^2\right)\) to \(\widetilde{\Omega}\left(1/\epsilon\right)\), decreasing the order of \(1/\epsilon\) by one in both cases, and the results hold for whatever sampling scheme is applied. Furthermore, the result of \(N=\widetilde{\Omega}\left(1/\epsilon\right)\) successfully answers the question of why a small number such as \(N=16,32,64\) is sufficient for RFLAF to perform well in real data, a phenomenon substantiated in empirical validations of RFLAF in \citep{ma2024random}.

\paragraph{Interpretation of the bounds} 
The fact that the ordinary RF models and RFLAF share the same bounds on \(s\) seems to be surprising and counterfactual at first glance, since the hypothesis class represented by RFLAF is intrinsically larger than the plain RF models. RFLAF is a bilinear model with a learnable kernel, while the plain RF model is simply a linear model. However, our theory provides the intuition to answer this question. Theorems \ref{thm31} and \ref{thm36} decompose the error estimates into three terms. The first quantity with respect to \(\lambda\) controls the random feature number and the sampling error. The second and third terms \(O(1/h\sqrt{n})+\epsilon\) control the sample size and the approximation error of the learned activation. Consequently, the controls on \(s\) and \(n\) are disentangled. The increased complexity induced by the bilinear structure is reflected mainly in the term of Rademacher complexity \(O(1/h\sqrt{n})\), irrelavent of the first term. Hence, it is possible to apply the weighted sampling scheme for a fewer number of random features in RFLAF. The cost of the learnable activation component in RF models is transferred to the requirement of the sample size \(n\), which increases from \(\Omega(1/\epsilon^2)\) to \(\Omega(1/\epsilon^2h^2)\).

\begin{figure*}[tbp]
\centering
\vspace{.3in}
\includegraphics[width=\linewidth]{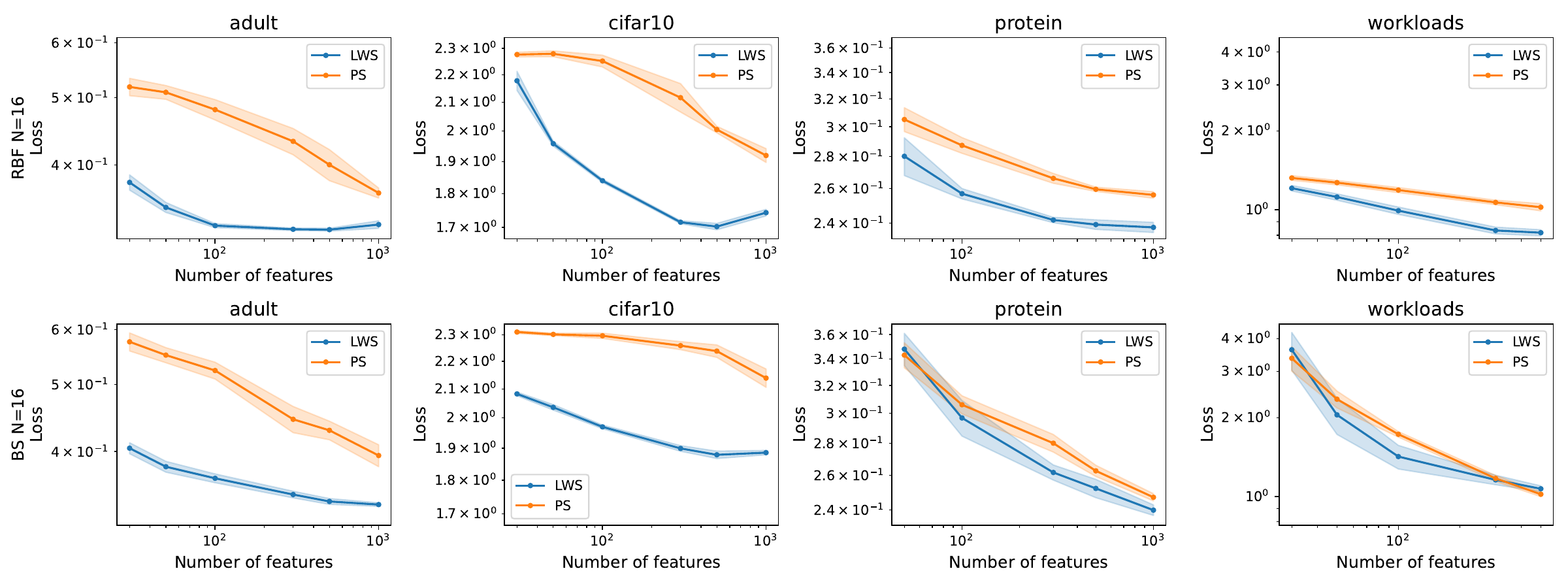}
\vspace{.3in}
\caption{\centering{Comparison of Leverage Weighted and Plain Sampling Schemes\\ for RFLAF of Different Basis Functions (RBF in Line 1, B-spline in Line 2)}}
\label{fig1}
\end{figure*}

\paragraph{Limitation of the results}

In \citep{li2021towards}, the author provides results for both the worst-case analysis and the refined analysis, where the latter provides a sharper bound on sample size \(n\) by utilizing the spectrum structure of the kernel. Although our result improves the bound on \(n\) in the squared error loss case, the refined analysis that might deduce a sharper bound on \(n\) remains to be explored. The major difficulty consists in the estimate of the local Rademacher complexity \citep{10.1214/009053605000000282} of RFLAF instead of the global Rademacher complexity in our work. Specifically, it requires that the hypothesis class \(\mathcal{F}_V\) to be convex with respect to the joint parameter \((a, v)\), which is generally not the case for a bilinear function class. Consequently, the derivation of the local Rademacher complexity of \(\mathcal{F}_V\) becomes intrinsically difficult. However, we suspect that the bound on \(n\) is not completely tight in this work, and theoretical techniques remain to be developed to address this problem.

\section{ALGORITHM}
\label{algorithm}

Previous results have shown that it is possible to reduce the width of RFLAF for faster computation by applying leverage weighted sampling. However, the problem is that the ridge leverage score function and the effective dimension of the approximate kernel are usually unknown beforehand. To address this problem, we propose a three-step procedure for learning the weighted RFLAF, presented in Algorithm \ref{alg1}. The algorithm first finds an approximate kernel by solving a low-rank matrix sensing problem (\ref{matrixsensing}) (line 2), which is studied in \citep{pmlr-v75-li18a,pmlr-v48-tu16}. And then it applies the approximate leverage weighted sampling in \citep{li2021towards} to obtain a fewer number of random features (line 3\(-\)5). Finally, it solves the main problem (\ref{mzer}) with the newly sampled features (line 6). We illustrate the algorithm for the squared error loss, and the case of Lipschitz continuous loss is obtained by substituting the losses in lines 2 and 6. 
% Experiments in the next section show that plain gradient descent methods are effective of searching for the minimizers of the optimization problem (\ref{mzer}) and (\ref{matrixsensing}). 
Below we deduce the learning risks of the algorithm in both cases of losses. We denote \(\widehat{\mathbf{K}}=\frac{1}{s}\mathbf{Z}_p(\tilde{a})\mathbf{Z}_p(\tilde{a})^\top\), and detailed proofs are provided in Appendix \ref{algoproof}.

\begin{figure*}[tbp]
\vspace{.3in}
\centering
\includegraphics[width=\linewidth]{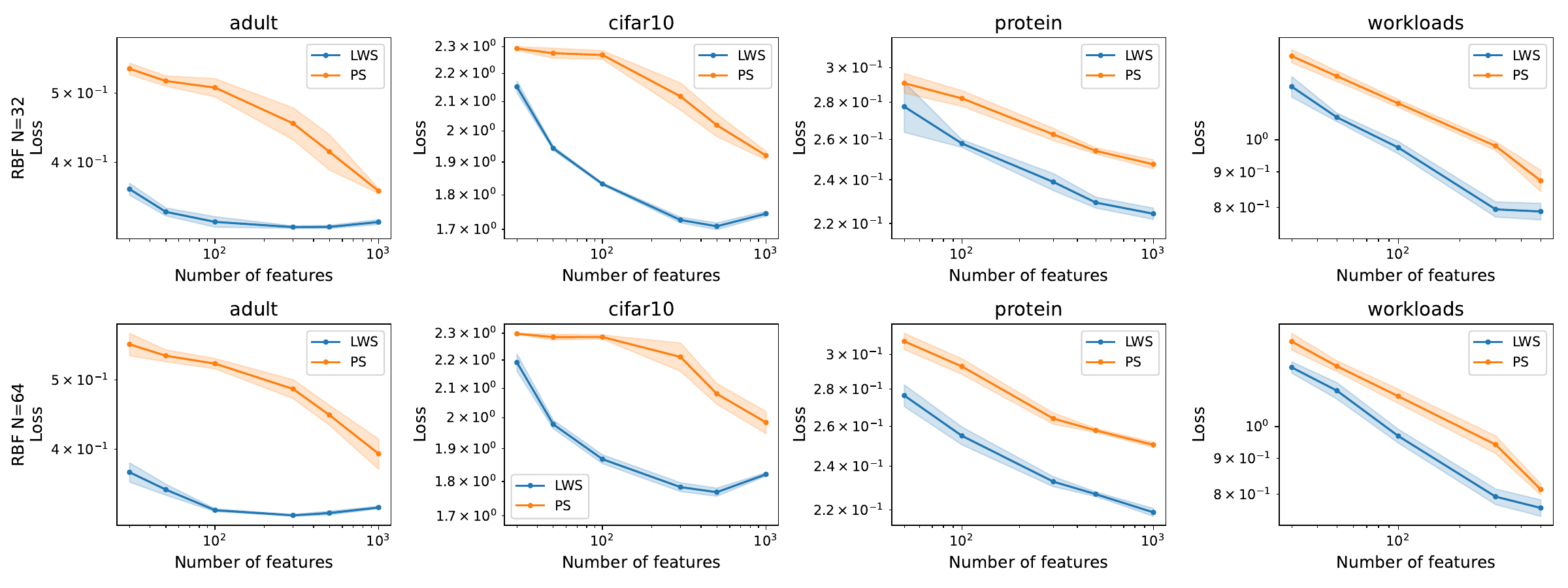}
\vspace{.3in}
\caption{\centering{Comparison of Leverage Weighted and Plain Sampling Schemes\\ for RFLAF of Different Grid Numbers ($N=32$ in Line 1, $N=64$ in Line 2)}}
\label{fig2}
\end{figure*}

\begin{theorem}[Squared Error Loss]
    \label{thm41}
    Under the Assumption \ref{assump22}, \ref{assump1}, for any \(\epsilon>0\), we set  
    \(\frac{1}{h}=\Omega\left(\left(\frac{1}{\epsilon}\log{\frac{1}{\epsilon}}\right)^{\frac{1}{2}}\right)\), \(N=\Omega\left(\frac{1}{\epsilon}{\left(\log{\frac{1}{\epsilon}}\right)^\frac{3}{2}}\right)\). Let the minimizer in line 6 of Algorithm \ref{alg1} with regularization parameter \(\lambda^*\) be \(\hat{f}\). For all \(\delta \in (0,1)\), if
    \[
        s \geq \frac{5{\|\tilde{\sigma}\|_{\infty}^2}}{\lambda}\log{\frac{16d_{\widetilde{\mathbf{K}}}^{\lambda}}{\delta}},\quad S\geq 5d_{\widehat{\mathbf{K}}}^{\lambda^*}\log{\frac{16d_{\widehat{\mathbf{K}}}^{\lambda^*}}{\delta}},
    \]
    then with probability \(1-\delta\), the excess risk of \(\hat{f}\) can be upper bounded by
    \[
        \mathbb{E}[l_{\hat{f}}]-\mathbb{E}[l_{f^*}]\leq 4\lambda + 4\lambda^* + O\left(\frac{1}{h\sqrt{n}}\right) + 4\epsilon.
    \]
\end{theorem}
\begin{theorem}[Lipschitz Continuous Loss]
  \label{thm42}
    Under the Assumption \ref{assump22}, \ref{lipsassump}, for any \(\epsilon>0\), we set  
    \(\frac{1}{h}=\Omega\left(\frac{1}{\epsilon}\left(\log{\frac{1}{\epsilon}}\right)^{\frac{1}{2}}\right)\), \(N=\Omega\left(\frac{1}{\epsilon^2}{\left(\log{\frac{1}{\epsilon}}\right)^\frac{3}{2}}\right)\). Let the minimizer in line 6 of Algorithm \ref{alg1} with regularization parameter \(\lambda^*\) be \(\hat{f}\). For all \(\delta \in (0,1)\), if
    \[
        s \geq \frac{5{\|\tilde{\sigma}\|_{\infty}^2}}{\lambda}\log{\frac{16d_{\widetilde{\mathbf{K}}}^{\lambda}}{\delta}},\quad S\geq 5d_{\widehat{\mathbf{K}}}^{\lambda^*}\log{\frac{16d_{\widehat{\mathbf{K}}}^{\lambda^*}}{\delta}},
    \]
    then with probability \(1-\delta\), the excess risk of \(\hat{f}\) can be upper bounded by
    \[
        \mathbb{E}[l_{\hat{f}}]-\mathbb{E}[l_{f^*}]\leq 4\sqrt{\lambda} + 4\sqrt{\lambda^*} + O\left(\frac{1}{h\sqrt{n}}\right) + 2\epsilon.
    \]
\end{theorem}

\paragraph{Discussion of the time and space complexity} The theoretical results indicate that \(n=\Omega(1/\epsilon^2h^2)\) and \(s = \tilde{\Omega}(1/\epsilon^2)\) in the most general Lipschitz loss case. From Table \ref{summse} and \ref{sumlip} we conclude that \(S < s\ll n\). In the training phase, suppose the epoch number of step 1 or 3 is \(E\). Then for step 1 and 3, the computational complexity is \(O(En\cdot \max\{d,N\}\cdot s)\) and \(O(En\cdot \max\{d,N\}\cdot S)\) respectively. For step 2, the computational cost is dominated by line 3 of \(O(ns^2+s^3)\). The dominating space complexity is \(O(ds)\) for storing the initial feature matrix \(W\) and \(O(s^2)\) for computing the inverse matrix in line 3 of Algorithm \ref{alg1}. In the inference phase, the computational complexity of the model is reduced from \(O(\max\{d,N\}\cdot s)\) to \(O(\max\{d,N\}\cdot S)\). Our method mainly lowers the time complexity in the inference time by lowering the width of the model.

\section{NUMERICAL RESULTS}
\label{simulation}

% to .

In this section, we will compare leverage weighted sampling scheme (\texttt{LWS}) to the baseline method plain sampling scheme (\texttt{PS}) where \(p(w)\) is set to be the standard normal distribution \(\mathcal{N}(0,I_d)\). We test the algorithm in two different settings with four real datasets. We consider two regression tasks \texttt{protein} and \texttt{workloads} for the squared error loss case, and two classification tasks \texttt{adult} and CIFAR-10 for the Lipschitz loss case. For classification tasks, the values of cross-entropy losses are reported as evaluation metric in figures, and we also supplement the results of accuracies in the appendix. In the first experimental setting, we compare \texttt{LWS} to \texttt{PS} within RFLAFs of two different basis functions: RBF and B-spline \citep{fakhoury2022exsplinet}. In the second setting, we compare \texttt{LWS} to \texttt{PS} within RFLAFs of different grid numbers \(N \in \{16,32,64\}\). In each experiment, we test the algorithm with the number of features \(S\in\{30, 50, 100, 300, 500, 1000\}\), and the number of the larger pool of random features in \texttt{LWS} is set to be \(s\in\{3000,5000\}\). We conduct each experiment with 8 different random seeds and illustrate the curves of the mean values with confidence intervals. More information can be found in Appendix \ref{expdetail}.

Figure \ref{fig1} shows the test losses of RFLAFs with RBF and B-spline with \(N=16\). It shows that \texttt{LWS} effectively reduces the width of RFLAF while preserving performances as those of \texttt{PS}. Consequently, the inference time of RFLAF can be shortened significantly. For instance, in the task CIFAR-10, the weighted RFLAF of width \(S=100\) achieves better test loss than the plain RFLAF of width \(S=1000\), improving the computational time by \(90\%\). In all four tasks and in both cases of RBF and B-spline (BS), \texttt{LWS} shows consistent advantages over \texttt{PS}. Figure \ref{fig2} further substantiates the result with RFLAF of different grid numbers \(N\). We test the algorithm on models of \(N\in\{32,64\}\). The results are similar to those of \(N=16\). Models using \texttt{LWS} always achieve lower test losses than \texttt{PS} within the same feature number, providing further evidence that \texttt{LWS} is efficient in finding features of directions more aligned with the data geometry.

% No experiments for n because the current bound is derived from the Rademacher complexity bound of the hypothesis class, which is not entirely tight due to the aforementioned analysis.

% The optimal activation (optimal kernel/optimal feature mapping) may not be unique. But whichever one we find, if we apply weighted sampling, it always work.

% loss
% sgd
% use gaussian for plain sampling

% what pool number

\section{CONCLUSION}

This paper combines the leverage weighted sampling to the model RFLAF and proposes an algorithm that successfully achieves superior performances with a substantially fewer number of random features. We provide rigorous analyses that deduce the sharpest bounds by far on the required number of features. Extensive experimental results validate our methods. Our work sheds light on the statistical limit of the high dimensional system. We highlight that a finer analysis of the sample size may be a potential future direction.

% \subsubsection*{Acknowledgements}
% All acknowledgments go at the end of the paper, including thanks to reviewers who gave useful comments, to colleagues who contributed to the ideas, and to funding agencies and corporate sponsors that provided financial support. 
% To preserve the anonymity, please include acknowledgments \emph{only} in the camera-ready papers. The acknowledgements do not count against the 9-page page limit in the camera-ready.

\bibliography{example_paper}
\bibliographystyle{apalike}

%%%%%%%%%%%%%%%%%%%%%%%%%%%%%%%%%%%%%%%%%%%%%%%%%%%%%%%%%%%%
\section*{Checklist}

% % %%% BEGIN INSTRUCTIONS %%%
% The checklist follows the references. For each question, choose your answer from the three possible options: Yes, No, Not Applicable.  You are encouraged to include a justification to your answer, either by referencing the appropriate section of your paper or providing a brief inline description (1-2 sentences). 
% Please do not modify the questions.  Note that the Checklist section does not count towards the page limit. Not including the checklist in the first submission won't result in desk rejection, although in such case we will ask you to upload it during the author response period and include it in camera ready (if accepted).

% \textbf{In your paper, please delete this instructions block and only keep the Checklist section heading above along with the questions/answers below.}
% % %%% END INSTRUCTIONS %%%

\begin{enumerate}

  \item For all models and algorithms presented, check if you include:
  \begin{enumerate}
    \item A clear description of the mathematical setting, assumptions, algorithm, and/or model. [Yes]
    \item An analysis of the properties and complexity (time, space, sample size) of any algorithm. [Yes]
    \item (Optional) Anonymized source code, with specification of all dependencies, including external libraries. [Yes]
  \end{enumerate}

  \item For any theoretical claim, check if you include:
  \begin{enumerate}
    \item Statements of the full set of assumptions of all theoretical results. [Yes]
    \item Complete proofs of all theoretical results. [Yes]
    \item Clear explanations of any assumptions. [Yes]     
  \end{enumerate}

  \item For all figures and tables that present empirical results, check if you include:
  \begin{enumerate}
    \item The code, data, and instructions needed to reproduce the main experimental results (either in the supplemental material or as a URL). [Yes]
    \item All the training details (e.g., data splits, hyperparameters, how they were chosen). [Yes]
    \item A clear definition of the specific measure or statistics and error bars (e.g., with respect to the random seed after running experiments multiple times). [Yes]
    \item A description of the computing infrastructure used. (e.g., type of GPUs, internal cluster, or cloud provider). [Yes]
  \end{enumerate}

  \item If you are using existing assets (e.g., code, data, models) or curating/releasing new assets, check if you include:
  \begin{enumerate}
    \item Citations of the creator If your work uses existing assets. [Not Applicable]
    \item The license information of the assets, if applicable. [Not Applicable]
    \item New assets either in the supplemental material or as a URL, if applicable. [Not Applicable]
    \item Information about consent from data providers/curators. [Not Applicable]
    \item Discussion of sensible content if applicable, e.g., personally identifiable information or offensive content. [Not Applicable]
  \end{enumerate}

  \item If you used crowdsourcing or conducted research with human subjects, check if you include:
  \begin{enumerate}
    \item The full text of instructions given to participants and screenshots. [Not Applicable]
    \item Descriptions of potential participant risks, with links to Institutional Review Board (IRB) approvals if applicable. [Not Applicable]
    \item The estimated hourly wage paid to participants and the total amount spent on participant compensation. [Not Applicable]
  \end{enumerate}

\end{enumerate}

\clearpage
\appendix
\thispagestyle{empty}

% Supplementary material: To improve readability, you must use a single-column format for the supplementary material.
\onecolumn
\aistatstitle{Supplementary Materials}

\section{RELATED WORK}
\label{relatedwork}

\paragraph{Scalable Kernel Methods and Random Features} The kernel method is a powerful technique for complex pattern recognition in machine learning tasks, including classification, regression, and dimensionality reduction. A key limitation, however, is its computational complexity. The method inherently requires calculation and storage that scales quadratically with the dataset size. This bottleneck, which involves constructing, storing, and inverting the kernel matrix, becomes prohibitive for large-scale applications. This challenge has spurred considerable research into scalable approximations and their theoretical analysis \citep{10.5555/2789272.2912104,rudi2017falkon,rudi2015less,alaoui2015fast}. Among the scalable kernel methods, the random feature models \citep{rahimi2008weighted,rahimi2007random} are one of the significant pillars for fast kernel approximation. Its foundation lies in the duality between a kernel function and its Fourier spectral density, allowing it to act as a scalable surrogate for kernel methods \citep{sun2018but,jacot2018neural,arora2019exact,zandieh2021scaling,du2019graph,zambon2020graph,fu2024can,shen2019random}. This has proven highly effective in areas like Transformer architecture, where it approximates kernels in the attention blocks to improve efficiency \citep{choromanski2020rethinking,peng2021random, fu2024can,guo2025schoenbat}. From a theoretical standpoint, the model, conceptualized as a two-layer network with frozen initial weights, provides a foundational framework for interpreting deep learning phenomena \citep{cao2019generalization,arora2019fine,mei2022generalization,chizat2020implicit}, spurring considerable research into its theoretical properties.

\paragraph{Feature complexity of Random Features Models} Theoretical topics for random feature model research have been continuously explored in recent years, particularly regarding the number of features required to match exact kernel performance. Existing work follows two main threads. One studies kernel matrix approximation \citep{rahimi2007random,sriperumbudur2015optimal,10.5555/3020847.3020936}, but often requires $\Omega(n)$ features, offering no computational benefit ($n$ being the sample number). The other directly analyzes generalization, yet early results still need $\Omega(n)$ features or exclude key methods like kernel ridge regression (KRR) \citep{rahimi2008weighted,bach2017equivalence}. \citet{avron2017random} showed $o(n)$ features suffice for empirical risk in KRR using modified sampling, though only for Gaussian kernels and without expected risk rates. \citet{rudi2017generalization} established that $\Omega(\sqrt{n} \log n)$ features achieve minimax-optimal error in KRR, but their analysis relies on strong assumptions and does not extend to SVMs or logistic regression. \citet{sun2018but} gave SVM bounds under low-noise conditions, but these scale exponentially in dimension and require optimized features. \citet{li2021towards} establish the first unified risk analysis for learning with random Fourier features under both squared error and Lipschitz continuous loss functions, deriving the sharpest bounds on the required number of features by far. Their bounds reveal a problem-specific trade-off between computational cost and expected risk convergence rate, which is characterized by the regularization parameter and the number of effective degrees of freedom.

\paragraph{Learnable Activation Functions} Research on parameterizable learnable activation functions constitutes a major thread in this field. \citep{tavakoli2021splash} parameterized continuous piecewise linear activation functions, numerically learning their parameters to improve accuracy and robustness against adversarial perturbations, and compared the performance of their SPLASH framework with various activation functions such as ReLU, Leaky ReLU \citep{maas2013rectifier}, PReLU \citep{he2015delving}, tanh, and sigmoid. Similarly, in \citep{zhou2021learning}, the authors parameterized the activation function as a piecewise linear unit, learning its parameters to optimize performance across different tasks. \citep{banerjee2019empirical} proposed an empirical method for learning variants of ReLU. \citep{bubeck2021law} studied two-layer neural networks, proposing a condition on the Lipschitz constant for polynomial activation functions under which the network can perfectly fit the data. They linked this condition to model parameter scale and robustness, and numerically analyzed the relationship between the number of ReLUs in the model and its robustness. There have also been various attempts using basis functions for activation functions, such as spline functions \citep{liu2024kan, liu2024kan2, fakhoury2022exsplinet, bohra2020learning, aziznejad2019deep}, polynomials \citep{goyal2019learning, goyal2020improved}, or neural networks \citep{zhang2022neural}.

\section{REMARK ON THE EXISTENCE OF THE MINIMIZER}

The problem formulated in (\ref{mainpro}) is non-convex. The reason is that RFLAF is a bilinear model that is not convex regarding the joint parameters \((a,v)\). Moreover, the function class presented by RFLAF is not a convex set. Hence, analyses regarding this problem are non-trivial.

To start with, we consider the existence of the minimizer of this problem. The solution to this problem is well-defined. For a fixed $a$, 
% let \(\hat{\sigma}(z)=\sum_{k=1}^{N}a_kB_k(z)\) be the parametrized activation function, then \(\hat{\sigma}(XW)\) is the nonlinear feature matrix. T
the inner minimization problem is a standard ridge regression problem.
\begin{equation*}
    v_a = \mathop{\mathrm{argmin}}_{v} \frac{1}{n}\left\|\mathbf{Z}_q({a})v-y\right\|^2+\lambda s \|v\|_2^2.
\end{equation*}
Hence, the inner problem has a closed-form solution given by
\begin{equation*}
    v_a = \left(\mathbf{Z}_q({a})^\top\mathbf{Z}_q({a})+\lambda ns \mathbf{I}\right)^{-1}\mathbf{Z}_q({a})^\top y.
\end{equation*}

The outer minimization problem then reduces to
\begin{equation}
    \label{outp}
    \min_{\|a\|\leq R}\left\{\frac{1}{n}\left\|\sum_{k=1}^{N}a_kB_k(XW)Qv_a-y\right\|^2+\lambda s \|v_a\|_2^2\right\}.
\end{equation}
% where \(Q=\mathrm{diag}\left\{\sqrt{p(w_1)/q(w_1)},...,\sqrt{p(w_s)/q(w_s)}\right\}\) contains weights for every feature.

This is still a nonconvex optimization problem with respect to the parameter \(a\). However, we note that \(v_a\) is continuous with respect to the variable \(a\). Hence, the target function in Eq. \ref{outp} is continuous with respect to \(a\). Consequently, within the compact set \(\left\{a\in\mathbb{R}^N:\|a\|\leq R\right\}\), the minimizer of this problem exists. 

\section{UPPER BOUNDS ON \(d_{\widetilde{\mathbf{K}}}^{\lambda}\)}
\label{effdim}

Explicitly, we have
\begin{equation}
    \label{summat}
    d_{\widetilde{\mathbf{K}}}^{\lambda} = \sum_{m=1}^{n} \frac{\lambda_m}{\lambda_m+\lambda},
\end{equation}
where \(\lambda_1\geq \cdots \geq \lambda_n\) are eigenvalues of the normalized Gram matrix \(\widetilde{\mathbf{K}}/n\) by Assumption \ref{assump22}. The deduction in Section 4.2 of \citep{bach2017equivalence} shows that there exists a constant \(C\) independent of \(n\) such that
\begin{equation}
    \label{tailbound}
    d_{\widetilde{\mathbf{K}}}^{\lambda}\leq C\max\{m:\lambda_m\geq \lambda\}.
\end{equation}
We consider four cases of \(d_{\widetilde{\mathbf{K}}}^{\lambda}\). For the first three cases, we use (\ref{tailbound}) to derive an upper bound on \(d_{\widetilde{\mathbf{K}}}^{\lambda}\):

(1) \(\widetilde{\mathbf{K}}/n\) is of finite rank. In this case, there exists a constant \(r\) such that \(\lambda_r>0\) and \(\lambda_m=0\) for all \(m>r\). Hence, for all \(\lambda\), we have \(d_{\widetilde{\mathbf{K}}}^{\lambda}\leq Cr=O(1)\).

(2) The eigenvalues of \(\widetilde{\mathbf{K}}/n\) decay in an exponential way, i.e., there exists \(A\in(0,1)\) and \(c_1>0\) such that \(\lambda_m\leq c_1A^m\) for all \(m\geq 1\). In this case, we have \(d_{\widetilde{\mathbf{K}}}^{\lambda}\leq O(\log(1/\lambda))\).

(3) The eigenvalues of \(\widetilde{\mathbf{K}}/n\) decay in a polynomial way, i.e., there exists \(t>1\) and \(c_2>0\) such that \(\lambda_m\leq c_2m^{-t}\) for all \(m\geq 1\). In this case, we have \(d_{\widetilde{\mathbf{K}}}^{\lambda}\leq O((1/\lambda)^{1/t})\).

For the last case, we use (\ref{summat}) to derive an upper bound on \(d_{\widetilde{\mathbf{K}}}^{\lambda}\):

(4) The eigenvalues of \(\widetilde{\mathbf{K}}/n\) decay harmonically, i.e., there exists a constant \(c_3>0\) such that \(\lambda_m\leq c_3/m\) for all \(m\geq 1\). In this case, we have \[d_{\widetilde{\mathbf{K}}}^{\lambda}< \sum_{m=1}^{n} \frac{\lambda_m}{\lambda}\leq \sum_{m=1}^{n} \frac{c_3}{m\lambda}= O\left(\frac{1}{\lambda}\log{n}\right).\]

\section{PROOFS OF THE MAIN THEOREMS}
\subsection{Proof of Theorem \ref{thm31}}
\label{proof31}

\begin{proof}
    Let the target function be \(f^*(x) = \mathbb{E}_{w}\left[\sigma(w^\top x)v(w)\right]\). We decompose the excess risk as follows.
    \[
        \mathbb{E}[l_{\hat{f}}]-\mathbb{E}[l_{f^*}]=\mathbb{E}[l_{\hat{f}}]-\mathbb{E}_n[l_{\hat{f}}]+\mathbb{E}_n[l_{\hat{f}}]-\mathbb{E}[l_{f^*}].
    \]
    The first two terms together can be bounded by simply using the Rademacher complexity of the hypothesis class which eventually induces a \(O(1/(h\sqrt{n}))\) term as is shown in \ref{rade}.
    
    For the remaining terms, we introduce an intermediate term to tackle them. Let the approximations of the activation function and the target function introduced in (\ref{tilde}) be \(\tilde{\sigma}(z)=\sum_{k=1}^{N}\tilde{a}_kB_k(z)\) and \(\tilde{f}(x) := \mathbb{E}_{w}\left[\tilde{\sigma}(w^\top x)v(w)\right]\), then \(\|\tilde{f}(x)-f^*(x)\|_{\infty}\leq \epsilon\). Let 
    \[
        f(x;\tilde{a},v) = \sum_{k=1}^{N}\tilde{a}_kB_k(x^\top W)Qv,
    \]
    and the optimal parameter be
    \[
        \tilde{v} = \mathop{\mathrm{argmin}}_{v\in\mathbb{R}^s}\frac{1}{n}\sum_{i=1}^{n} \left(f(x_i;\tilde{a},v)-y_i\right)^2+\lambda s \|v\|_2^2,
    \]
    and denote \(f_{\tilde{a}}(x)=f(x;\tilde{a},\tilde{v})\).

    By the definition of (\ref{mzer}), we naturally have the following inequalities.
    \[
        \mathbb{E}_n[l_{\hat{f}}]\leq \mathbb{E}_n[l_{\hat{f}}]+\lambda s \|\hat{v}\|_2^2 \leq \mathbb{E}_n[l_{{f}_{\tilde{a}}}]+\lambda s \|\tilde{v}\|_2^2.
    \]
    Then we have
    \[
        \begin{split}
            &\,\mathbb{E}[l_{\hat{f}}]-\mathbb{E}[l_{f^*}]\leq(\mathbb{E}[l_{\hat{f}}]-\mathbb{E}_n[l_{\hat{f}}])+(\mathbb{E}_n[l_{{f}_{\tilde{a}}}]+\lambda s \|\tilde{v}\|_2^2-\mathbb{E}[l_{f^*}]).
        \end{split}
    \]
    
    For the second term, we have
    \[
        \begin{split}
            &\,\mathbb{E}_n[l_{{f}_{\tilde{a}}}]+\lambda s \|\tilde{v}\|_2^2-\mathbb{E}[l_{f^*}]\\
            =&\, \frac{1}{n}\sum_{i=1}^{n} \left(f_{\tilde{a}}(x_i)-y_i\right)^2+\lambda s \|\tilde{v}\|_2^2-\mathbb{E}[l_{f^*}]\\
            =&\, \frac{1}{n}\sum_{i=1}^{n} \left(f_{\tilde{a}}(x_i)-f^*(x_i)-\varepsilon_i\right)^2+\lambda s \|\tilde{v}\|_2^2-\mathbb{E}[l_{f^*}]\\
            =&\, \frac{1}{n}\sum_{i=1}^{n} \left[\left(f_{\tilde{a}}(x_i)-f^*(x_i)\right)^2-2\left(f_{\tilde{a}}(x_i)-f^*(x_i)\right)\varepsilon_i+\varepsilon_i^2\right]+\lambda s \|\tilde{v}\|_2^2-\mathbb{E}[l_{f^*}]\\
            =&\, \frac{1}{n}\sum_{i=1}^{n} \left(f_{\tilde{a}}(x_i)-f^*(x_i)\right)^2+\lambda s \|\tilde{v}\|_2^2-\frac{2}{n}\sum_{i=1}^{n}\left(f_{\tilde{a}}(x_i)-f^*(x_i)\right)\varepsilon_i+\frac{1}{n}\sum_{i=1}^{n}\varepsilon_i^2-\mathbb{E}[l_{f^*}]\\
        \end{split}
    \]
    Hence, we have
    \[
        \begin{split}
            &\,\mathbb{E}_n[l_{{f}_{\tilde{a}}}]+\lambda s \|\tilde{v}\|_2^2-\mathbb{E}[l_{f^*}]\\
            =&\, \min_{v\in\mathbb{R}^s}\frac{1}{n}\sum_{i=1}^{n} \left(f(x_i;\tilde{a},v)-y_i\right)^2+\lambda s \|{v}\|_2^2-\mathbb{E}[l_{f^*}]\\
            =&\, \min_{v\in\mathbb{R}^s}\frac{1}{n}\sum_{i=1}^{n} \left(f(x_i;\tilde{a},v)-f^*(x_i)\right)^2+\lambda s \|{v}\|_2^2-\frac{2}{n}\sum_{i=1}^{n}\left(f(x_i;\tilde{a},v)-f^*(x_i)\right)\varepsilon_i+\frac{1}{n}\sum_{i=1}^{n}\varepsilon_i^2-\mathbb{E}[l_{f^*}]\\
        \end{split}
    \]
    Moreover, we have that
    \[
        \begin{split}
            &\, \min_{v\in\mathbb{R}^s}\frac{1}{n}\sum_{i=1}^{n} \left(f(x_i;\tilde{a},v)-f^*(x_i)\right)^2+\lambda s \|{v}\|_2^2-\frac{2}{n}\sum_{i=1}^{n}\left(f(x_i;\tilde{a},v)-f^*(x_i)\right)\varepsilon_i+\frac{1}{n}\sum_{i=1}^{n}\varepsilon_i^2-\mathbb{E}[l_{f^*}]\\
            \leq&\, \min_{v\in\mathbb{R}^s}\Bigg\{\frac{1}{n}\sum_{i=1}^{n}\left[ 2(f(x_i;\tilde{a},v)-\tilde{f}(x_i))^2 + 2(\tilde{f}(x_i)-f^*(x_i))^2 \right]+2\lambda s \|{v}\|_2^2\\
            &\,\quad \quad ~~~-\frac{2}{n}\sum_{i=1}^{n}(f(x_i;\tilde{a},v)-\tilde{f}(x_i))\varepsilon_i-\frac{2}{n}\sum_{i=1}^{n}(\tilde{f}(x_i)-f^*(x_i))\varepsilon_i+\left(\frac{1}{n}\sum_{i=1}^{n}\varepsilon_i^2-\mathbb{E}[l_{f^*}]\right)\Bigg\}\\
            =&\, \min_{v\in\mathbb{R}^s}\Bigg\{2\left(\frac{1}{n}\sum_{i=1}^{n}(f(x_i;\tilde{a},v)-\tilde{f}(x_i))^2+\lambda s \|{v}\|_2^2\right)-\frac{2}{n}\sum_{i=1}^{n}(f(x_i;\tilde{a},v)-\tilde{f}(x_i))\varepsilon_i\Bigg\} \\
            &\,\quad \quad ~~~+ \frac{2}{n}\sum_{i=1}^{n}(\tilde{f}(x_i)-f^*(x_i))^2-\frac{2}{n}\sum_{i=1}^{n}(\tilde{f}(x_i)-f^*(x_i))\varepsilon_i+\left(\frac{1}{n}\sum_{i=1}^{n}\varepsilon_i^2-\mathbb{E}[l_{f^*}]\right),\\
        \end{split}
    \]
    where we used \((a+b)^2\leq 2a^2+2b^2\) in the inequality.
    
    From now on, we analyze the above four terms one by one. We start from the first term with minimum.

    Let
    \(
        \tilde{\mathbf{f}}=[\tilde{f}(x_1),...,\tilde{f}(x_n)]^\top,
    \)
    then
    \[
    \begin{split}
        \tilde{v}_0 =&\, \mathop{\mathrm{argmin}}_{v\in\mathbb{R}^s} \frac{1}{n}\sum_{i=1}^{n}(f(x_i;\tilde{a},v)-\tilde{f}(x_i))^2+\lambda s \|{v}\|_2^2\\
        =&\,\mathop{\mathrm{argmin}}_{v\in\mathbb{R}^s} \frac{1}{n}\|\mathbf{Z}_q(\tilde{a})v - \tilde{\mathbf{f}}\|_2^2+\lambda s \|{v}\|_2^2.
    \end{split}
    \]
    Therefore, we have 
    \[
        \begin{split}
            &\, \min_{v\in\mathbb{R}^s}\Bigg\{2\left(\frac{1}{n}\sum_{i=1}^{n}(f(x_i;\tilde{a},v)-\tilde{f}(x_i))^2+\lambda s \|\tilde{v}\|_2^2\right)-\frac{2}{n}\sum_{i=1}^{n}(f(x_i;\tilde{a},v)-\tilde{f}(x_i))\varepsilon_i\Bigg\} \\
            % &\,\quad \quad ~~~+ \frac{2}{n}\sum_{i=1}^{n}\left(\tilde{f}(x_i)-f^*(x_i)\right)^2-\frac{2}{n}\sum_{i=1}^{n}\left(\tilde{f}(x_i)-f^*(x_i)\right)\varepsilon_i+\left(\frac{1}{n}\sum_{i=1}^{n}\varepsilon_i^2-\mathbb{E}[l_{f^*}]\right)\\
            \leq &\, 2\left(\frac{1}{n}\sum_{i=1}^{n}(f(x_i;\tilde{a},\tilde{v}_0)-\tilde{f}(x_i))^2+\lambda s \|\tilde{v}_0\|_2^2\right)-\frac{2}{n}\sum_{i=1}^{n}(f(x_i;\tilde{a},\tilde{v}_0)-\tilde{f}(x_i))\varepsilon_i
        \end{split}
    \]
    For the first term, by applying Lemma \ref{lemma22}, we have that with high probability
    \[
        \frac{1}{n}\sum_{i=1}^{n}(f(x_i;\tilde{a},\tilde{v}_0)-\tilde{f}(x_i))^2+\lambda s \|\tilde{v}_0\|_2^2 \leq 2\lambda.
    \]
    For the second term, let \(\boldsymbol{\varepsilon}=(\varepsilon_1,...,\varepsilon_n)\) we have that
    \[
    \begin{split}
        &\,\frac{2}{n}\sum_{i=1}^{n}(f(x_i;\tilde{a},\tilde{v}_0)-\tilde{f}(x_i))\varepsilon_i=\frac{2}{n}\langle \mathbf{Z}_q(\tilde{a})\tilde{v}_0-\tilde{\mathbf{f}},\boldsymbol{\varepsilon} \rangle.
    \end{split}
    \]
    The minimizer has an analytic expression as
    \[
    \begin{split}
        \tilde{v}_0 =&\,\frac{1}{s}(\widehat{\mathbf{K}}_q(\tilde{a})+n\lambda\mathbf{I})^{-1}\mathbf{Z}_q(\tilde{a})^\top\tilde{\mathbf{f}}\\
        =&\,\frac{1}{s}\mathbf{Z}_q(\tilde{a})^\top(\widehat{\mathbf{K}}_q(\tilde{a})+n\lambda\mathbf{I})^{-1}\tilde{\mathbf{f}},
    \end{split}
    \]
    where the second equality follows from the Woodbury inversion lemma.
    As a result, we have
    \[
        \mathbf{Z}_q(\tilde{a})\tilde{v}_0=\frac{1}{s}\mathbf{Z}_q(\tilde{a})\mathbf{Z}_q(\tilde{a})^\top(\frac{1}{s}\mathbf{Z}_q(\tilde{a})\mathbf{Z}_q(\tilde{a})^\top+n\lambda\mathbf{I})^{-1}\tilde{\mathbf{f}}=\widehat{\mathbf{K}}_q(\tilde{a})(\widehat{\mathbf{K}}_q(\tilde{a})+n\lambda\mathbf{I})^{-1}\tilde{\mathbf{f}}.
    \]
    Hence, we have
    \[
    \begin{split}
        &\,-\frac{2}{n}\sum_{i=1}^{n}(f(x_i;\tilde{a},\tilde{v}_0)-\tilde{f}(x_i))\varepsilon_i\\
        =&\,-\frac{2}{n}\langle \widehat{\mathbf{K}}_q(\tilde{a})(\widehat{\mathbf{K}}_q(\tilde{a})+n\lambda\mathbf{I})^{-1}\tilde{\mathbf{f}}-\tilde{\mathbf{f}},\boldsymbol{\varepsilon} \rangle\\
        =&\,-\frac{2}{n}\langle \widehat{\mathbf{K}}_q(\tilde{a})(\widehat{\mathbf{K}}_q(\tilde{a})+n\lambda\mathbf{I})^{-1}\tilde{\mathbf{f}}-(\widehat{\mathbf{K}}_q(\tilde{a})+n\lambda\mathbf{I})(\widehat{\mathbf{K}}_q(\tilde{a})+n\lambda\mathbf{I})^{-1}\tilde{\mathbf{f}},\boldsymbol{\varepsilon} \rangle\\
        =&\,2\lambda \langle (\widehat{\mathbf{K}}_q(\tilde{a})+n\lambda\mathbf{I})^{-1}\tilde{\mathbf{f}},\boldsymbol{\varepsilon} \rangle
    \end{split}
    \]
    Note that \(\langle (\widehat{\mathbf{K}}_q(\tilde{a})+n\lambda\mathbf{I})^{-1}\tilde{\mathbf{f}},\boldsymbol{\varepsilon} \rangle\) is actually a Gaussian random variable whose variance is \[\langle (\widehat{\mathbf{K}}_q(\tilde{a})+n\lambda\mathbf{I})^{-1}\tilde{\mathbf{f}},(\widehat{\mathbf{K}}_q(\tilde{a})+n\lambda\mathbf{I})^{-1}\tilde{\mathbf{f}} \rangle = \tilde{\mathbf{f}}^\top(\widehat{\mathbf{K}}_q(\tilde{a})+n\lambda\mathbf{I})^{-2}\tilde{\mathbf{f}}\]
    The proof of Lemma \ref{lemma22} shows that 
    \[
        \tilde{\mathbf{f}}^\top(\widehat{\mathbf{K}}_q(\tilde{a})+n\lambda\mathbf{I})^{-2}\tilde{\mathbf{f}}=\frac{1}{n^2\lambda^2}\|\mathbf{Z}_q(\tilde{a})\tilde{v}_0 - \tilde{\mathbf{f}}\|_2^2\leq \frac{2\lambda}{n\lambda^2}=\frac{2}{n\lambda}.
    \]
    Hence, with high probability, we have that
    \[
        2\lambda \langle (\widehat{\mathbf{K}}_q(\tilde{a})+n\lambda\mathbf{I})^{-1}\tilde{\mathbf{f}},\boldsymbol{\varepsilon} \rangle = O\left(\sqrt{\frac{\lambda}{n}}\right)=O\left(\frac{1}{\sqrt{n}}\right).
    \]
    The last equality holds due to that fact that \(\lambda = o(1)\) always holds in practice.

    % !!! Start from here 3,4,5. Not finished yet. tonight we do it.

    Now, we consider the rest of the terms. Because \(\|\tilde{f}(x)-f^*(x)\|_{\infty}\leq \epsilon\), we have
    \[
        \frac{2}{n}\sum_{i=1}^{n}(\tilde{f}(x_i)-f^*(x_i))^2\leq 2\epsilon^2.
    \]

    Since \(\varepsilon_i\) are Gaussian and \(\tilde{f}-f^*\) is bounded, we have that \((\tilde{f}(x_i)-f^*(x_i))\varepsilon_i\) are Gaussian random variables with bounded variance. Using Hoeffding's Lemma or Central Limit Theorem, we have that with high probability
    \[
        \frac{2}{n}\sum_{i=1}^{n}(\tilde{f}(x_i)-f^*(x_i))\varepsilon_i = O\left(\frac{1}{\sqrt{n}}\right).
    \]

    % and used \(\|\tilde{f}-f^*\|_{\infty}\leq \epsilon\) in the second inequality.

    Given that \(y_i = f^*(x_i)+\varepsilon_i\) where \(\varepsilon_i\) are i.i.d. noises, we have 
    \[
        \mathbb{E}[l_{f^*}]=\mathbb{E}[\varepsilon_i^2]=\sigma^2.
    \]
    Since \(\mathbb{E}[\varepsilon_i^4]<\infty\), by Central Limit Theorem, with high probability we have that
    \[
        \frac{1}{n}\sum_{i=1}^{n}\varepsilon_i^2-\mathbb{E}[l_{f^*}] = O\left(\frac{1}{\sqrt{n}}\right).
    \]

    Combining all the estimates, we conclude that with high probability, it happens that
    \[
        \mathbb{E}[l_{\hat{f}}]-\mathbb{E}[l_{f^*}] \leq 4\lambda + O\left(\frac{1}{h\sqrt{n}}\right) + 2\epsilon^2.
    \]
    Substituting the \(\epsilon\) term with \(\sqrt{\epsilon}\), we obtain the desired results in the theorem.

\end{proof}

\subsection{Proof of Theorem \ref{thm36}}
\label{proof36}
\begin{proof}
    Similar to the proof of Theorem \ref{thm31}, we decompose the excess risk as follows.
    \[
        \begin{split}
            \mathbb{E}[l_{\hat{g}}]-\mathbb{E}[l_{g^*}]=\mathbb{E}[l_{\hat{g}}]-\mathbb{E}_n[l_{\hat{g}}]+\mathbb{E}_n[l_{\hat{g}}]-\mathbb{E}[l_{g^*}].
        \end{split}
    \]

    Let the approximations of the activation function and the target function introduced in (\ref{tilde}) be \(\tilde{\sigma}(z)=\sum_{k=1}^{N}\tilde{a}_kB_k(z)\) and \(\tilde{g}(x) := \mathbb{E}_{w}\left[\tilde{\sigma}(w^\top x)v(w)\right]\), then \(\left\|\tilde{g}(x)-g^*(x)\right\|_{\infty}\leq \epsilon\).

    Next, we introduce an intermediate term. Let 
    \[
        g(x;\tilde{a},v) = \sum_{k=1}^{N}\tilde{a}_kB_k(x^\top W)Qv,
    \]
    and the optimal parameter be
    \[
        \tilde{v} = \mathop{\mathrm{argmin}}_{v\in\mathbb{R}^s}\frac{1}{n}\sum_{i=1}^{n} l\left(g(x_i;\tilde{a},v),y_i\right)+\lambda s \|v\|_2^2,
    \]
    and denote \(g_{\tilde{a}}(x)=g(x;\tilde{a},\tilde{v})\).

    For the third term, by the definition of (\ref{mzer2}), we have that 
    \[
    \begin{split}
        \mathbb{E}_n[l_{\hat{g}}]&\leq \mathbb{E}_n[l_{\hat{g}}] + \lambda s \|\hat{v}\|_2^2\leq \mathbb{E}_n[l_{g_{\tilde{a}}}] + \lambda s \|\tilde{v}\|_2^2.
    \end{split}
    \]
    Plugging it back to the decomposition, we have that
    \[
    \begin{split}
        &\,\mathbb{E}[l_{\hat{g}}]-\mathbb{E}[l_{g^*}]\\
        =&\,\mathbb{E}[l_{\hat{g}}]-\mathbb{E}_n[l_{\hat{g}}]+\mathbb{E}_n[l_{\hat{g}}]-\mathbb{E}[l_{g^*}]\\
        \leq&\,\mathbb{E}[l_{\hat{g}}]-\mathbb{E}_n[l_{\hat{g}}]+\mathbb{E}_n[l_{g_{\tilde{a}}}] + \lambda s \|\tilde{v}\|_2^2-\mathbb{E}[l_{g^*}]\\
        =&\,(\mathbb{E}[l_{\hat{g}}]-\mathbb{E}_n[l_{\hat{g}}])+(\mathbb{E}_n[l_{g_{\tilde{a}}}]+\lambda s \|\tilde{v}\|_2^2-\mathbb{E}_n[l_{\tilde{g}}]) +(\mathbb{E}_n[l_{\tilde{g}}]-\mathbb{E}[l_{\tilde{g}}])+(\mathbb{E}[l_{\tilde{g}}]-\mathbb{E}[l_{g^*}]).
    \end{split}
    \]

    We analyze the four terms in the above expression separately.

    For the first term, by applying Lemma \ref{rade}, we have
    \[
        \mathbb{E}[l_{\hat{g}}]-\mathbb{E}_n[l_{\hat{g}}] = O\left(\frac{1}{h\sqrt{n}}\right).
    \]

    % For the second term, the proof of lemma 22 suggests that
    % \[
    %     \lambda s \|\tilde{v}\|_2^2\leq 2\lambda.
    % \]

    For the second term, we have
    \[
    \begin{split}
        &\,\mathbb{E}_n[l_{g_{\tilde{a}}}]+\lambda s \|\tilde{v}\|_2^2-\mathbb{E}_n[l_{\tilde{g}}]\\
        =&\, \min_{v\in \mathbb{R}^s}\frac{1}{n}\sum_{i=1}^{n} l\left(g(x_i;\tilde{a},v),y_i\right) - \frac{1}{n}\sum_{i=1}^{n} l\left(\tilde{g}(x_i),y_i\right)+\lambda s \|v\|_2^2\\
        \leq &\, \min_{v\in \mathbb{R}^s} \frac{1}{n}\sum_{i=1}^{n} \left|g(x_i;\tilde{a},v)-\tilde{g}(x_i)\right|+\lambda s \|v\|_2^2\\
        \leq  &\, \min_{v\in \mathbb{R}^s} \sqrt{\frac{1}{n}\sum_{i=1}^{n} \left|g(x_i;\tilde{a},v)-\tilde{g}(x_i)\right|^2}+\lambda s \|v\|_2^2\\
        \leq &\, \min_{v\in \mathbb{R}^s} \sqrt{\frac{1}{n}\sum_{i=1}^{n} \left|g(x_i;\tilde{a},v)-\tilde{g}(x_i)\right|^2+\lambda s \|v\|_2^2}+\frac{1}{n}\sum_{i=1}^{n} \left|g(x_i;\tilde{a},v)-\tilde{g}(x_i)\right|^2+\lambda s \|v\|_2^2.\\
    \end{split}
    \]
    % Because the function \(\sqrt{x}+x\) and \(x\) share the same monotonicity over \(x\in[0,+\infty)\), the two functions have the same minimizer within the same domain. 
    By Lemma \ref{lemma22}, we have that 
    \[
        \min_{v\in \mathbb{R}^s} \frac{1}{n}\sum_{i=1}^{n} \left|g(x_i;\tilde{a},v)-\tilde{g}(x_i)\right|^2+\lambda s \|v\|_2^2 \leq 2\lambda.
    \]
    Consequently, it holds that
    \[
        \min_{v\in \mathbb{R}^s} \sqrt{\frac{1}{n}\sum_{i=1}^{n} \left|g(x_i;\tilde{a},v)-\tilde{g}(x_i)\right|^2+\lambda s \|v\|_2^2}+\frac{1}{n}\sum_{i=1}^{n} \left|g(x_i;\tilde{a},v)-\tilde{g}(x_i)\right|^2+\lambda s \|v\|_2^2\leq \sqrt{2\lambda} +2\lambda.
    \]
    As a result, considering \(\lambda = o(1)\), we have that
    \[
        \mathbb{E}_n[l_{g_{\tilde{a}}}]+\lambda s \|\tilde{v}\|_2^2-\mathbb{E}_n[l_{\tilde{g}}]\leq \sqrt{2\lambda} +2\lambda\leq 4\sqrt{\lambda}.
    \]

    For the third term, because \(\tilde{g}(x)\) is bounded and \(l\) is Lipschitz continuous with respect to the first variable, we have that \(l(g(x),y)\) is bounded. Hence, using Hoeffding's Lemma for bounded random variables, with high probability, we have that
    \[
        \mathbb{E}_n[l_{\tilde{g}}]-\mathbb{E}[l_{\tilde{g}}] = O\left(\frac{1}{\sqrt{n}}\right).
    \]

    For the fourth term, using the Lipschitz property of \(l\) and the approximation error between \(\tilde{g}\) and \(g^*\), we have that
    \[
    \begin{split}
        &\,\mathbb{E}[l_{\tilde{g}}]-\mathbb{E}[l_{g^*}]=\mathbb{E}[l(\tilde{g}(x),y)-l(g^*(x),y)]
        \leq \mathbb{E}\left[|\tilde{g}(x)-g^*(x)|\right]\leq \left\|\tilde{g}(x)-g^*(x)\right\|_{\infty}\leq\epsilon.
    \end{split}
    \]
    Combining all the above results, we have that
    \[
        \mathbb{E}[l_{\hat{g}}]-\mathbb{E}[l_{g^*}]\leq 4\sqrt{\lambda} + O\left(\frac{1}{h\sqrt{n}}\right) + \epsilon.
    \]
\end{proof}

\section{TECHNICAL LEMMAS}

\subsection{The Rademacher Complexity of the Hypothesis Class}

In this subsection, we provide an upper bound on the Rademacher complexity of the hypothesis class \(\mathcal{F}_V:=\left\{f(x)=\sum_{k=1}^{N}a_kB_k(x^\top W)Qv:~\|a\|_2\leq R,~v\in\mathbb{R}^s\right\}\) in which \(R = O(1/h\sqrt{N})\). Inspired by Lemma \ref{lemma22} that the optimal parameter \(\|\hat{v}\|_2^2\leq 2/s\), we only need to consider the following hypothesis class
\[
\mathcal{F}_B:=\left\{f(x)=\sum_{k=1}^{N}a_kB_k(x^\top W)Qv:~\|a\|_2\leq R,~\|v\|_2\leq 2/\sqrt{s}\right\}.
\]

For the given loss function \(l\) and samples \(S=((x_1,y_1),...,(x_n,y_n))\), we denote
\[
	l \circ \mathcal{F}_B \circ S:= \{(l(f(x_1),y_1),...,l(f(x_n),y_n)):f\in \mathcal{F}_B\}.
\]

\begin{lemma}
    Under the conditions of Lemma \ref{lemma21} and Assumption \ref{lipsassump}, with high probability over the sampled data, the Rademacher complexity of the hypothesis class \(l \circ \mathcal{F}_B\) with respect to the samples \(S\) is upper bounded as
    \[
        \mathcal{R}(l \circ \mathcal{F}_B \circ S)\leq O\left(\frac{1}{h\sqrt{n}}\right).
    \]
\end{lemma}

\begin{proof}
    Let \(K_1\) be the upper bounds of all possible \(\|a\|_2\) and \(K_2\) be the upper bounds of \(\|Qv\|_2\) in \(\mathcal{F}_B\) respectively. The proof of Lemma D.2 in \citep{ma2024random} shows that
    \[
        \mathcal{R}(l \circ \mathcal{F}_B \circ S)\leq K_1K_2\sqrt{Ns/n}.
    \]
    Since \(K_1 = O(1/h\sqrt{N})\) and \(K_2 = \max_{i\in[s]}\left\{\sqrt{\frac{p(w_i)}{q(w_i)}}\right\}\cdot \frac{2}{\sqrt{s}}\), we have that
    \[
        \mathcal{R}(l \circ \mathcal{F}_B \circ S)\leq O\left(\frac{1}{h\sqrt{n}}\right).
    \]
\end{proof}

Combining the lemma with Theorem 26.5 in \citep{shalev2014understanding}, we directly have
\begin{lemma}
    \label{rade}
    Under the conditions of Lemma \ref{lemma21} and Assumption \ref{lipsassump}, with high probability over the sampled data, for the minimizer \(\hat{g}\) in (\ref{mzer2}), it holds that
    \[
        \mathbb{E}[l_{\hat{g}}]-\mathbb{E}_n[l_{\hat{g}}]\leq 2\mathcal{R}(l \circ \mathcal{F}_B \circ S) + O\left(1/\sqrt{n}\right).
    \]
    Consequently, we have
    \[
        \mathbb{E}[l_{\hat{g}}]-\mathbb{E}_n[l_{\hat{g}}]\leq O\left(\frac{1}{h\sqrt{n}}\right).
    \]
\end{lemma}
We highlight that the case of MSE loss also satisfies Assumption \ref{lipsassump} when the target function is bounded. Hence, the result holds for both the MSE loss and the general Lipschitz loss.

\subsection{The Error Estimate on the Ridge Regression Problem}

In this subsection, we present and prove a critical error estimation of a particular ridge regression problem, which is a key component in the analyses of our main theorems. The results of this subsection are originally derived in Lemma 22 in \citep{li2021towards} and are restated or modified to our scenario here for the completeness of our theoretical analyses.

Let the approximations of the activation function and the target function introduced in (\ref{tilde}) be \(\tilde{\sigma}(z)=\sum_{k=1}^{N}\tilde{a}_kB_k(z)\) and \(\tilde{f}(x) := \mathbb{E}_{w}\left[\tilde{\sigma}(w^\top x)v(w)\right]\), then \(\|\tilde{f}(x)-f^*(x)\|_{\infty}\leq \epsilon\). The corresponding kernel with respect to the parameter \(\tilde{a}\) is  
\(
    \tilde{k}(x,x^\prime) = \mathbb{E}_{w}\left[\tilde{\sigma}(w^\top x)\tilde{\sigma}({w}^\top x^\prime)\right].
\)
We denote the RKHS with respect to \(\tilde{k}\) as \(\widetilde{\mathcal{H}} = \left\{f(x) = \mathbb{E}_{w}\left[\tilde{\sigma}(w^\top x)v(w)\right]~|~v:\mathbb{R}^d\rightarrow \mathbb{R}\right\}\). The Gram matrix is \(\widetilde{\mathbf{K}}=\mathbb{E}_{w}\left[\tilde{\sigma}(Xw)\tilde{\sigma}(Xw)^\top\right].\)

We consider the weighted RFLAF in (\ref{Wrflaf}) and denote the model in the ridge regression problem as \(f(x;\tilde{a},v)=\sum_{i=1}^N \tilde{a}_i B_i(x^\top W)Qv\), where \(v\) is the learnable parameter. The ridge regression problem that we aim to analyze then is
\begin{equation*}
    \begin{split}
        \tilde{v}_0 =&\, \mathop{\mathrm{argmin}}_{v\in\mathbb{R}^s} \frac{1}{n}\sum_{i=1}^{n}(f(x_i;\tilde{a},v)-\tilde{f}(x_i))^2+\lambda s \|{v}\|_2^2.
        % =&\,\mathop{\mathrm{argmin}}_{v\in\mathbb{R}^s} \frac{1}{n}\left\|\mathbf{Z}_q(\tilde{a})v - \tilde{\mathbf{f}}\right\|^2+\lambda s \|{v}\|_2^2.
    \end{split}
\end{equation*}
Let the in-sample labels be 
\(
    \tilde{\mathbf{f}}=[\tilde{f}(x_1),...,\tilde{f}(x_n)]^\top,
\)
then we can also write the problem in a compact vector form as
\begin{equation}
    \label{comrid}
    \begin{split}
        \tilde{v}_0 =\mathop{\mathrm{argmin}}_{v\in\mathbb{R}^s} \frac{1}{n}\|\mathbf{Z}_q(\tilde{a})v - \tilde{\mathbf{f}}\|_2^2+\lambda s \|{v}\|_2^2.
    \end{split}
\end{equation}
We denote the hypothesis space of the optimization problem as \(\widehat{\mathcal{H}}\). Therefore, \(\widehat{\mathcal{H}}\) is the finite-width approximation of \(\widetilde{\mathcal{H}}\). 

Finally, we recall that the ridge leverage score function is
\(
    l_{\lambda}(w) := p(w)\tilde{\sigma}(Xw)^\top( \widetilde{\mathbf{K}}+n\lambda \mathbf{I})^{-1}\tilde{\sigma}(Xw).
\)
And the effective number of parameters with respect to kernel \(\widetilde{\mathbf{K}}\) and regularization parameter \(\lambda\) is
\(
    d_{\widetilde{\mathbf{K}}}^{\lambda}:=\mathrm{Tr}[\widetilde{\mathbf{K}}(\widetilde{\mathbf{K}}+n\lambda \mathbf{I})^{-1}]=\int_{\mathcal{W}}l_{\lambda}(w)dw.
\)
Below we formally state the lemma.
% !!! \(\tilde{a}\) to \(\widetilde{\mathcal{H}}\), \(\widehat{\mathcal{H}}\) to be the hypothesis space sorresponding to the optimization problem.

% list and prove lemma 22. states lemma 27 and lemma 29 only for reference. (if necessary, prove them again here.)

\begin{lemma}
    \label{lemma22}
    Under the conditions of Lemma \ref{lemma21} and the Assumptions \ref{assump22}, \ref{assump1}, let \(\tilde{l}:\mathbb{R}^s \rightarrow \mathbb{R}\) be a measurable function such that for all \(w\in \mathbb{R}^s\), \( \tilde{l}(w) \geq l_{\lambda}(w)\) with \(d_{\tilde{l}}=\int_{\mathbb{R}^s} \tilde{l}(w)dv < \infty\). Suppose that \(\{w_i\}_{i=1}^s\) are sampled independently from the probability density function \(q(w)=\tilde{l}(w)/d_{\tilde{l}}\). For any \(\tilde{f} \in \widetilde{\mathcal{H}}\) with \(\|\tilde{f}\|_{\widetilde{\mathcal{H}}}\leq 1\) and for all \(\delta \in (0,1)\), if
    \[
        s \geq 5d_{\tilde{l}}\log{\frac{16d_{\widetilde{\mathbf{K}}}^{\lambda}}{\delta}},
    \]
    then with probability greater than \(1-\delta\), the following inequality holds
    \[
        \min_{v}\left\{\frac{1}{n}\|\mathbf{Z}_q(\tilde{a})v - \tilde{\mathbf{f}}\|^2+\lambda s \|{v}\|_2^2\right\}\leq 2\lambda.
    \]
\end{lemma}

\begin{proof}
The minimizer has an analytic expression as
\[
\begin{split}
    \tilde{v}_0 =&\,\frac{1}{s}(\widehat{\mathbf{K}}_q(\tilde{a})+n\lambda\mathbf{I})^{-1}\mathbf{Z}_q(\tilde{a})^\top\tilde{\mathbf{f}}\\
    =&\,\frac{1}{s}\mathbf{Z}_q(\tilde{a})^\top(\widehat{\mathbf{K}}_q(\tilde{a})+n\lambda\mathbf{I})^{-1}\tilde{\mathbf{f}},
\end{split}
\]
where the second equality follows from the Woodbury inversion lemma.

Hence, we have that
\[
    \begin{split}
        \frac{1}{n}\|\mathbf{Z}_q(\tilde{a})\tilde{v}_0 - \tilde{\mathbf{f}}\|_2^2
        =&\,\frac{1}{n}\|\widehat{\mathbf{K}}_q(\tilde{a})(\widehat{\mathbf{K}}_q(\tilde{a})+n\lambda\mathbf{I})^{-1}\tilde{\mathbf{f}} - \tilde{\mathbf{f}}\|_2^2\\
        =&\,\frac{1}{n}\|n\lambda(\widehat{\mathbf{K}}_q(\tilde{a})+n\lambda\mathbf{I})^{-1}\tilde{\mathbf{f}}\|_2^2\\
        =&\, n\lambda^2 \tilde{\mathbf{f}}^\top(\widehat{\mathbf{K}}_q(\tilde{a})+n\lambda\mathbf{I})^{-2}\tilde{\mathbf{f}}
    \end{split}
\]

Moreover, for the regularizer term, we have
\[
    \begin{split}
        \lambda s \|\tilde{v}_0\|_2^2 =&\, \lambda s \|\frac{1}{s}\mathbf{Z}_q(\tilde{a})^\top(\widehat{\mathbf{K}}_q(\tilde{a})+n\lambda\mathbf{I})^{-1}\tilde{\mathbf{f}}\|_2^2\\
        =&\, \frac{\lambda}{s}\tilde{\mathbf{f}}^\top(\widehat{\mathbf{K}}_q(\tilde{a})+n\lambda\mathbf{I})^{-1}\mathbf{Z}_q(\tilde{a})\mathbf{Z}_q(\tilde{a})^\top(\widehat{\mathbf{K}}_q(\tilde{a})+n\lambda\mathbf{I})^{-1}\tilde{\mathbf{f}}\\
        =&\, \lambda\tilde{\mathbf{f}}^\top(\widehat{\mathbf{K}}_q(\tilde{a})+n\lambda\mathbf{I})^{-1}(\widehat{\mathbf{K}}_q(\tilde{a})+n\lambda\mathbf{I})(\widehat{\mathbf{K}}_q(\tilde{a})+n\lambda\mathbf{I})^{-1}\tilde{\mathbf{f}}-n\lambda^2\tilde{\mathbf{f}}^\top(\widehat{\mathbf{K}}_q(\tilde{a})+n\lambda\mathbf{I})^{-2}\tilde{\mathbf{f}}\\
        =&\, \lambda\tilde{\mathbf{f}}^\top(\widehat{\mathbf{K}}_q(\tilde{a})+n\lambda\mathbf{I})^{-1}\tilde{\mathbf{f}}-n\lambda^2\tilde{\mathbf{f}}^\top(\widehat{\mathbf{K}}_q(\tilde{a})+n\lambda\mathbf{I})^{-2}\tilde{\mathbf{f}}\\
    \end{split}
\]

As a result, we have that
\[
\begin{split}
    \min_{v}\left\{\frac{1}{n}\|\mathbf{Z}_q(\tilde{a})v - \tilde{\mathbf{f}}\|^2+\lambda s \|{v}\|_2^2\right\}
    =&\,\frac{1}{n}\|\mathbf{Z}_q(\tilde{a})\tilde{v}_0 - \tilde{\mathbf{f}}\|_2^2 + \lambda s \|\tilde{v}_0\|_2^2\\
    =&\, \lambda\tilde{\mathbf{f}}^\top(\widehat{\mathbf{K}}_q(\tilde{a})+n\lambda\mathbf{I})^{-1}\tilde{\mathbf{f}}
\end{split}
\]

Note that \(\widetilde{\mathbf{K}}=\mathbb{E}_{w\sim q}[\widehat{\mathbf{K}}_q(a)]\), we deduce that
\[
\begin{split}
    &\, \lambda\tilde{\mathbf{f}}^\top(\widehat{\mathbf{K}}_q(\tilde{a})+n\lambda\mathbf{I})^{-1}\tilde{\mathbf{f}}=\lambda\tilde{\mathbf{f}}^\top(\widetilde{\mathbf{K}}+n\lambda\mathbf{I}+\widehat{\mathbf{K}}_q(\tilde{a})-\widetilde{\mathbf{K}})^{-1}\tilde{\mathbf{f}}\\
    =&\,\lambda\tilde{\mathbf{f}}^\top(\widetilde{\mathbf{K}}+n\lambda\mathbf{I})^{-\frac{1}{2}}\left(\mathbf{I}+(\widetilde{\mathbf{K}}+n\lambda\mathbf{I})^{-\frac{1}{2}}(\widehat{\mathbf{K}}_q(\tilde{a})-\widetilde{\mathbf{K}})(\widetilde{\mathbf{K}}+n\lambda\mathbf{I})^{-\frac{1}{2}}\right)^{-1}(\widetilde{\mathbf{K}}+n\lambda\mathbf{I})^{-\frac{1}{2}}\tilde{\mathbf{f}}\\
\end{split}
\]

By Lemma \ref{lemma27}, it follows that if
\[
    s \geq d_{\tilde{l}}(\frac{1}{\epsilon^2}+\frac{2}{3\epsilon})\log{\frac{16d_{\widetilde{\mathbf{K}}}^{\lambda}}{\delta}},
\]
then \((\widetilde{\mathbf{K}}+n\lambda\mathbf{I})^{-\frac{1}{2}}(\widehat{\mathbf{K}}_q(\tilde{a})-\widetilde{\mathbf{K}})(\widetilde{\mathbf{K}}+n\lambda\mathbf{I})^{-\frac{1}{2}}\succeq-\epsilon \mathbf{I}\).

Setting \(\epsilon = 1/2\), we have that
\[
    \begin{split}
        &\, \lambda\tilde{\mathbf{f}}^\top(\widehat{\mathbf{K}}_q(\tilde{a})+n\lambda\mathbf{I})^{-1}\tilde{\mathbf{f}}\\
        =&\,\lambda\tilde{\mathbf{f}}^\top(\widetilde{\mathbf{K}}+n\lambda\mathbf{I})^{-\frac{1}{2}}\left(\mathbf{I}+(\widetilde{\mathbf{K}}+n\lambda\mathbf{I})^{-\frac{1}{2}}(\widehat{\mathbf{K}}_q(\tilde{a})-\widetilde{\mathbf{K}})(\widetilde{\mathbf{K}}+n\lambda\mathbf{I})^{-\frac{1}{2}}\right)^{-1}(\widetilde{\mathbf{K}}+n\lambda\mathbf{I})^{-\frac{1}{2}}\tilde{\mathbf{f}}\\
        \leq &\, (1-\epsilon)^{-1}\lambda\tilde{\mathbf{f}}^\top(\widetilde{\mathbf{K}}+n\lambda\mathbf{I})^{-1}\tilde{\mathbf{f}}\leq (1-\epsilon)^{-1}\lambda\tilde{\mathbf{f}}^\top\widetilde{\mathbf{K}}^{-1}\tilde{\mathbf{f}}=2\lambda\tilde{\mathbf{f}}^\top\widetilde{\mathbf{K}}^{-1}\tilde{\mathbf{f}}\leq 2\lambda,
    \end{split}
\]
where the last inequality is obtained using Lemma \ref{lemma29}.

Consequently, we have \(\min_{v}\left\{\frac{1}{n}\|\mathbf{Z}_q(\tilde{a})v - \tilde{\mathbf{f}}\|^2+\lambda s \|{v}\|_2^2\right\}\leq 2\lambda\).

\end{proof}

\subsubsection{Additional Lemmas for Proving Lemma \ref{lemma22}}

In the following we supplement some lemmas used in the proof of Lemma \ref{lemma22}. Lemma \ref{lemma27} describes the required number of random features to obtain an approximate Gram matrix. It is derived by utilizing the matrix Bernstein concentration inequality stated in Lemma \ref{lemma26}, and the proof is originated from Lemma 27 in \citep{li2021towards} which we also restate here. Lemma \ref{lemma29} corresponds to Lemma 29 in \citep{li2021towards}.

\begin{lemma}
    \label{lemma27}
    Suppose that the assumptions in Lemma \ref{lemma22} hold and let \(\epsilon\geq \sqrt{\frac{M}{s}}+\frac{2L}{3s}\) with constants \(M\) and \(L\) (explicitly defined in the proof). If
    \[
        s \geq d_{\tilde{l}}(\frac{1}{\epsilon^2}+\frac{2}{3\epsilon})\log{\frac{16d_{\widetilde{\mathbf{K}}}^{\lambda}}{\delta}},
    \]
    then for all \(\delta \in (0,1)\), with probability greater than \(1-\delta\), it holds that
    \[  
        -\epsilon\mathbf{I}\preceq (\widetilde{\mathbf{K}}+n\lambda\mathbf{I})^{-\frac{1}{2}}(\widehat{\mathbf{K}}_q(\tilde{a})-\widetilde{\mathbf{K}})(\widetilde{\mathbf{K}}+n\lambda\mathbf{I})^{-\frac{1}{2}}\preceq \epsilon\mathbf{I}.
    \]
\end{lemma}

\begin{proof}
    For brevity, we denote \(\widehat{\mathbf{K}}=\widehat{\mathbf{K}}_q(\tilde{a})\). Let
    \[
        \mathbf{R}_i = p(w_i)/q(w_i)\cdot(\widetilde{\mathbf{K}}+n\lambda\mathbf{I})^{-\frac{1}{2}}\tilde{\sigma}(Xw_i)\tilde{\sigma}(Xw_i)^\top(\widetilde{\mathbf{K}}+n\lambda\mathbf{I})^{-\frac{1}{2}}.
    \]
    Then we have
    \[
        \widehat{\mathbf{K}} = \frac{1}{s}\sum_{i=1}^s\mathbf{R}_i,\quad\text{and}\quad \widetilde{\mathbf{K}} =\mathbb{E}_{w_i\sim q}[\mathbf{R}_i].
    \]

    Consider the operator norm of \(\mathbf{R}_i\). Because \(\mathbf{R}_i\) is actually a rank one matrix, we have that the operator norm of this matrix is equal to its trace. Hence, we have
    \[  
        \begin{split}
            \|\mathbf{R}_i\|_2=&\,\mathrm{Tr}(p(w_i)/q(w_i)\cdot(\widetilde{\mathbf{K}}+n\lambda\mathbf{I})^{-\frac{1}{2}}\tilde{\sigma}(Xw_i)\tilde{\sigma}(Xw_i)^\top(\widetilde{\mathbf{K}}+n\lambda\mathbf{I})^{-\frac{1}{2}})\\
            =&\,p(w_i)/q(w_i)\cdot\mathrm{Tr}(\tilde{\sigma}(Xw_i)^\top(\widetilde{\mathbf{K}}+n\lambda\mathbf{I})^{-1}\tilde{\sigma}(Xw_i))\\
            =&\, l_\lambda(w_i)/q(w_i).
        \end{split}
    \]
    Consequently, \(L:=\sup_{i}l_\lambda(w_i)/q(w_i)\leq \sup_{i}\tilde{l}(w_i)/q(w_i)=d_{\tilde{l}}\) and \(\|\mathbf{R}_i\|_2\leq L\).

    Consider the variance of the matrix, we have
    \[  
        \begin{split}
            &\,\mathbf{R}_i\mathbf{R}_i^\top\\
            =&\, p(w_i)^2/q(w_i)^2\cdot(\widetilde{\mathbf{K}}+n\lambda\mathbf{I})^{-\frac{1}{2}}\tilde{\sigma}(Xw_i)\tilde{\sigma}(Xw_i)^\top(\widetilde{\mathbf{K}}+n\lambda\mathbf{I})^{-1}\tilde{\sigma}(Xw_i)\tilde{\sigma}(Xw_i)^\top(\widetilde{\mathbf{K}}+n\lambda\mathbf{I})^{-\frac{1}{2}}\\
            =&\,p(w_i)l_\lambda(w_i)/q(w_i)^2\cdot(\widetilde{\mathbf{K}}+n\lambda\mathbf{I})^{-\frac{1}{2}}\tilde{\sigma}(Xw_i)\tilde{\sigma}(Xw_i)^\top(\widetilde{\mathbf{K}}+n\lambda\mathbf{I})^{-\frac{1}{2}}\\
            \preceq &\,d_{\tilde{l}}\cdot p(w_i)/q(w_i)\cdot(\widetilde{\mathbf{K}}+n\lambda\mathbf{I})^{-\frac{1}{2}}\tilde{\sigma}(Xw_i)\tilde{\sigma}(Xw_i)^\top(\widetilde{\mathbf{K}}+n\lambda\mathbf{I})^{-\frac{1}{2}}.\\
        \end{split}
    \]
    Hence, we have
    \[  
        \mathbb{E}_{w\sim q}[\mathbf{R}_i\mathbf{R}_i^\top]=d_{\tilde{l}}(\widetilde{\mathbf{K}}+n\lambda\mathbf{I})^{-\frac{1}{2}}\widetilde{\mathbf{K}}(\widetilde{\mathbf{K}}+n\lambda\mathbf{I})^{-\frac{1}{2}}=:\mathbf{M}_1=\mathbf{M}_2:=\mathbb{E}_{w\sim q}[\mathbf{R}_i^\top\mathbf{R}_i].
    \]
    We then have the equalities
    \[  
    \begin{split}
        M=&\,\|\mathbf{M}_1\|_2 = d_{\tilde{l}}\frac{\lambda_1}{\lambda_1+\lambda}=d_{\tilde{l}}d_1,\\
        D=&\,\frac{\mathrm{Tr}(\mathbf{M}_1)+\mathrm{Tr}(\mathbf{M}_2)}{M}=2d_{\widetilde{\mathbf{K}}}^\lambda\frac{\lambda_1}{\lambda_1+\lambda}=2d_{\widetilde{\mathbf{K}}}^\lambda d_1^{-1}.
    \end{split}
    \]

    Now, we apply Lemma \ref{lemma26}, for \(\epsilon \geq \sqrt{M/s}+2L/3s\) and for all \(\delta\in(0,1)\), with probability \(1-\delta\), we have
    \[  
    \begin{split}
        P(\|\widehat{\mathbf{K}}-\widetilde{\mathbf{K}}\|_2\geq \epsilon) \leq&\, 4D\exp\left(-\frac{s\epsilon^2/2}{M+2L\epsilon/3}\right)\\
        \leq &\,8d_{\widetilde{\mathbf{K}}}^\lambda d_1^{-1}\exp\left(-\frac{s\epsilon^2/2}{d_{\tilde{l}}d_1+2d_{\tilde{l}}\epsilon/3}\right)\\
        \leq &\, 16d_{\widetilde{\mathbf{K}}}^\lambda \exp\left(-\frac{s\epsilon^2}{d_{\tilde{l}}(1+2\epsilon/3)}\right)\\
        \leq &\, \delta,
    \end{split}
    \]
    where we assume that \(\lambda \leq \lambda_1\) and consequently \(d_1\in[1/2,1)\).
    
\end{proof}

\begin{lemma}
    \label{lemma29}
    Let \(f \in \mathcal{H}\), where \(\mathcal{H}\) is the reproducing kernel Hilbert space (RKHS) with respect to a kernel \(k\). Let \(\mathbf{f}=[f(x_1),...,f(x_n)]^\top\) and \(\mathbf{K}\) be the Gram matrix with respect to the kernel \(k\) and the samples. If \(\|f\|_{\mathcal{H}}\leq 1\), then we have
    \[\mathbf{f}^\top\mathbf{K}^{-1}\mathbf{f}\leq 1.\]
\end{lemma}

\begin{proof}
    If we prove that \(\mathbf{f}\cdot\mathbf{f}^\top\preceq \mathbf{K}\), then \[\mathbf{f}^\top\mathbf{K}^{-1}\mathbf{f}=\mathrm{Tr}(\mathbf{f}^\top\mathbf{K}^{-1}\mathbf{f})=\mathrm{Tr}(\mathbf{f}\cdot\mathbf{f}^\top\mathbf{K}^{-1})\leq \mathrm{Tr}(\mathbf{f}\cdot\mathbf{f}^\top(\mathbf{f}\cdot\mathbf{f}^\top)^{-1})\leq 1,\]
    where we use the notation \((\mathbf{f}\cdot\mathbf{f}^\top)^{-1}\) to represent the pseudo-inverse of the rank one matrix \(\mathbf{f}\cdot\mathbf{f}^\top\), and the last inequality holds because \(\mathbf{f}\cdot\mathbf{f}^\top\) has rank one.

    Now we prove \(\mathbf{f}\cdot\mathbf{f}^\top\preceq \mathbf{K}\). For all \(\boldsymbol{e} = (e_1,...,e_n)\), we have that
    \[
        \begin{split}
            \boldsymbol{e}^\top\mathbf{f}\cdot\mathbf{f}^\top \boldsymbol{e}=&\,(\mathbf{f}^\top \boldsymbol{e})^2
            =\left(\sum_{i=1}^{n}e_if(x_i)\right)^2\\
            =&\,\left(\sum_{i=1}^{n}e_i\mathbb{E}_{w}[\sigma(w^\top x_i)v(w)]\right)^2\\
            =&\,\left(\mathbb{E}_{w}\left[\sum_{i=1}^{n}e_i\sigma(w^\top x_i)v(w)\right]\right)^2\\
            \leq&\,\mathbb{E}_{w}\left[\sum_{i=1}^{n}e_i\sigma(w^\top x_i)\right]^2\mathbb{E}_{w}\left[v(w)^2\right]\\
            \leq&\,\mathbb{E}_{w}\left[\sum_{i=1}^{n}e_i\sigma(w^\top x_i)\right]^2\\
            =&\,\boldsymbol{e}^\top\mathbf{K} \boldsymbol{e},
        \end{split}
    \]
    where the first inequality follows from the Cauchy-Schwarz inequality, and the second inequality holds because of the assumption that \(\mathbb{E}_{w}\left[v(w)\right]^2\leq 1\).
    
\end{proof}

\begin{lemma}[Matrix Bernstein inequality (e.g., Corollary 7.3.3 in \citep{tropp2015introduction})]
\label{lemma26}
Let \(\mathbf{R}\) be a fixed \(d_1\times d_2\) matrix over the set of real numbers. Suppose that \(\{\mathbf{R}_1,...\mathbf{R}_n\}\) is an independent and identically distributed sample of \(d_1\times d_2\) matrices such that
\[
    \mathbb{E}[\mathbf{R}_i]=\mathbf{R},\quad \|\mathbf{R}_i\|_2\leq L,
\]
where \(L > 0\) is a constant independent of the sample. Furthermore, let \(\mathbf{M}_1,\mathbf{M}_2\) be semi-definite upper bounds for the matrix-valued variances
\[
\begin{split}
    &\mathrm{Var}_1[\mathbf{R}_i]\preceq\mathbb{E}[\mathbf{R}_i\mathbf{R}_i^\top]\preceq \mathbf{M}_1,\\
    &\mathrm{Var}_2[\mathbf{R}_i]\preceq\mathbb{E}[\mathbf{R}_i^\top\mathbf{R}_i]\preceq \mathbf{M}_2.\\
\end{split}
\]
Let \(M=\max(\|\mathbf{M}_1\|_2,\|\mathbf{M}_2\|_2)\) and \(D=\frac{\mathrm{Tr}(\mathbf{M}_1)+\mathrm{Tr}(\mathbf{M}_2)}{M}\). Then, for \(\epsilon\geq \sqrt{M/n}+2L/3n\), we can bound
\[
    \bar{\mathbf{R}}_n=\frac{1}{n}\sum_{i=1}^n\mathbf{R}_i
\]
around its mean using the concentration inequality
\[
    P(\|\bar{\mathbf{R}}_n-\mathbf{R}\|_2\geq \epsilon)\leq 4D\exp\left(-\frac{n\epsilon^2/2}{M+2L\epsilon/3}\right).
\]
\end{lemma}

\section{PROOFS OF THE LEARNING RISKS OF THE ALGORITHM}
\label{algoproof}
In this part, we present the proof of Theorem \ref{thm41}. The proof of Theorem \ref{thm42} is parallel and therefore omitted.

\begin{proof}
Let the minimizer of (\ref{mzer}) in the plain sampling scheme be \(\hat{f}_0\). Theorem \ref{thm31} and Corollary \ref{cor2} show that for all \(\delta\in (0,1)\), with probability \(1-\delta\), if
\[
  s\geq \frac{5{\|\tilde{\sigma}\|_{\infty}^2}}{\lambda}\log{\frac{16d_{\widetilde{\mathbf{K}}}^{\lambda}}{\delta}},
\]
then the learning risk is bounded by
\begin{equation}
  \label{pse}
  \mathbb{E}[l_{\hat{f}_0}]\leq 4\lambda + O\left(\frac{1}{h\sqrt{n}}\right) + 2\epsilon + \mathbb{E}[l_{f^*}].
\end{equation}

Let the empirical Gram matrix in the plain sampling scheme be \(\widehat{\mathbf{K}}=\frac{1}{s}\mathbf{Z}_p(\tilde{a})\mathbf{Z}_p(\tilde{a})^\top\) where \(p(w)\) is the spectral density function and \(\tilde{a}\) is the optimal parameter in problem (\ref{mzer}). Let \(\tilde{\sigma}(x):=\sum_{i=1}^{N}\tilde{a}_iB_i(x)\) and define \(\hat{k}(x, y) = \frac{1}{s}\sum_{m=1}^{s}\tilde{\sigma}(w_m^\top x)\tilde{\sigma}(w_m^\top y) = \int\tilde{\sigma}(w^\top x)\tilde{\sigma}(w^\top y)d\hat{P}(w)\), where \(\hat{P}\) is the empirical measure on \(\{w_m\}_{m=1}^{s}\). Denote the reproducing kernel Hilbert space associated with kernel \(\hat{k}\) by \(\hat{\mathcal{H}}\). We now treat \(\hat{k}\) as the actual kernel to learn through leverage weighted sampling scheme. Let \(f_{\hat{\mathcal{H}}}\) be the function in the reproducing kernel Hilbert space \(\hat{\mathcal{H}}\) achieving the minimal risk, i.e., \(\mathbb{E}[l_{f_{\hat{\mathcal{H}}}}] = \inf_{f\in\hat{\mathcal{H}}}\mathbb{E}[l_f]\). The corresponding leverage score function then is
\[
    l_{\lambda}(w) := p(w)\tilde{\sigma}(Xw)^\top( \widehat{\mathbf{K}}+n\lambda \mathbf{I})^{-1}\tilde{\sigma}(Xw).
\]
Therefore, we have
\[
    l_{\lambda}(w_i) := p(w)[\tilde{\sigma}(XW)^\top( \widehat{\mathbf{K}}+n\lambda \mathbf{I})^{-1}\tilde{\sigma}(XW)]_{ii},
\]
where \([A]_{ii}\) denotes the \(i\)-th diagonal element of matrix \(A\). To sample according to the leverage score function, let \(p_i=[\tilde{\sigma}(XW)^\top( \widehat{\mathbf{K}}+n\lambda \mathbf{I})^{-1}\tilde{\sigma}(XW)]_{ii}\) and normalize them to \(q(w_i)=p_i/\sum_{i=1}^{s}p_i\). Then the algorithm resamples \(S\) features according to \(q(w)\) and solves the optimization problem (\ref{mzer}) with the weighted RFLAF. Let the minimizer in the weighted setting be \(\hat{f}\). By applying Theorem \ref{thm31} and Corollary \ref{cor1}, we have that for all \(\delta\in (0,1)\), with probability \(1-\delta\), if
\[
  S\geq 5d_{\widehat{\mathbf{K}}}^{\lambda^*}\log{\frac{16d_{\widehat{\mathbf{K}}}^{\lambda^*}}{\delta}},
\]
then the learning risk is bounded by
\begin{equation}
  \label{lwse}
  \mathbb{E}[l_{\hat{f}}]\leq 4\lambda^* + O\left(\frac{1}{h\sqrt{n}}\right) + 2\epsilon + \mathbb{E}[l_{f_{\hat{\mathcal{H}}}}].
\end{equation}
Because \(f_{\hat{\mathcal{H}}}\) is the minimizer achieving the minimal risk over \(\hat{\mathcal{H}}\), we have that \(\mathbb{E}[l_{f_{\hat{\mathcal{H}}}}]\leq \mathbb{E}[l_{\hat{f}_0}]\). By combining Eqs. (\ref{pse}) and (\ref{lwse}), we obtain the desired result.

\end{proof}

\section{SUPPLEMENTARY DETAILS OF THE EXPERIMENTS}
\label{expdetail}

\paragraph{Benchmark datasets} The datasets used to validate the algorithm can be found in the following sources.

\begin{table}[htbp]
	\caption{Sources of the datasets.}
	\label{urltable}
	\vskip 0.15in
	\begin{center}
	\begin{small}
	\begin{tabular}{l|p{13cm}}
	\toprule
	Datasets & Sources \\
	\midrule
    CIFAR-10 & \texttt{torchvision.datasets} in Python \\
	adult    &  \url{https://archive.ics.uci.edu/ml/machine-learning-databases/adult/adult.data}\\
	protein    & \url{https://archive.ics.uci.edu/ml/machine-learning-databases/00265/CASP.csv} \\
	workloads      &  \url{https://archive.ics.uci.edu/ml/machine-learning-databases/00493/datasets.zip}\\
	\bottomrule
	\end{tabular}
	\end{small}
	\end{center}
	\vskip -0.1in
\end{table}

\paragraph{Other details} The optimization problem of (\ref{mzer}) and (\ref{matrixsensing}) can be solved using common gradient descent methods. In our experiments, we simply use stochastic gradient descent (SGD) with learning rate 1e-3. The models can be implemented within the framework of PyTorch. And we run experiments on GeForce RTX 3090 GPU for faster computation through CUDA. Each experiment is repeated with 8 different seeds, namely 8 different sets of initial random features, and the results of mean values are illustrated in the pictures with $80\%$ confidence intervals. 

\paragraph{Figures of accuracies for the classification tasks} We show Figure \ref{fig3} and \ref{fig4} with accuracy as the evaluation metric corresponding to Figure \ref{fig1} and \ref{fig2} respectively for the classification tasks.

\begin{figure*}[htbp]
\centering
\vspace{.3in}
\includegraphics[width=0.6\linewidth]{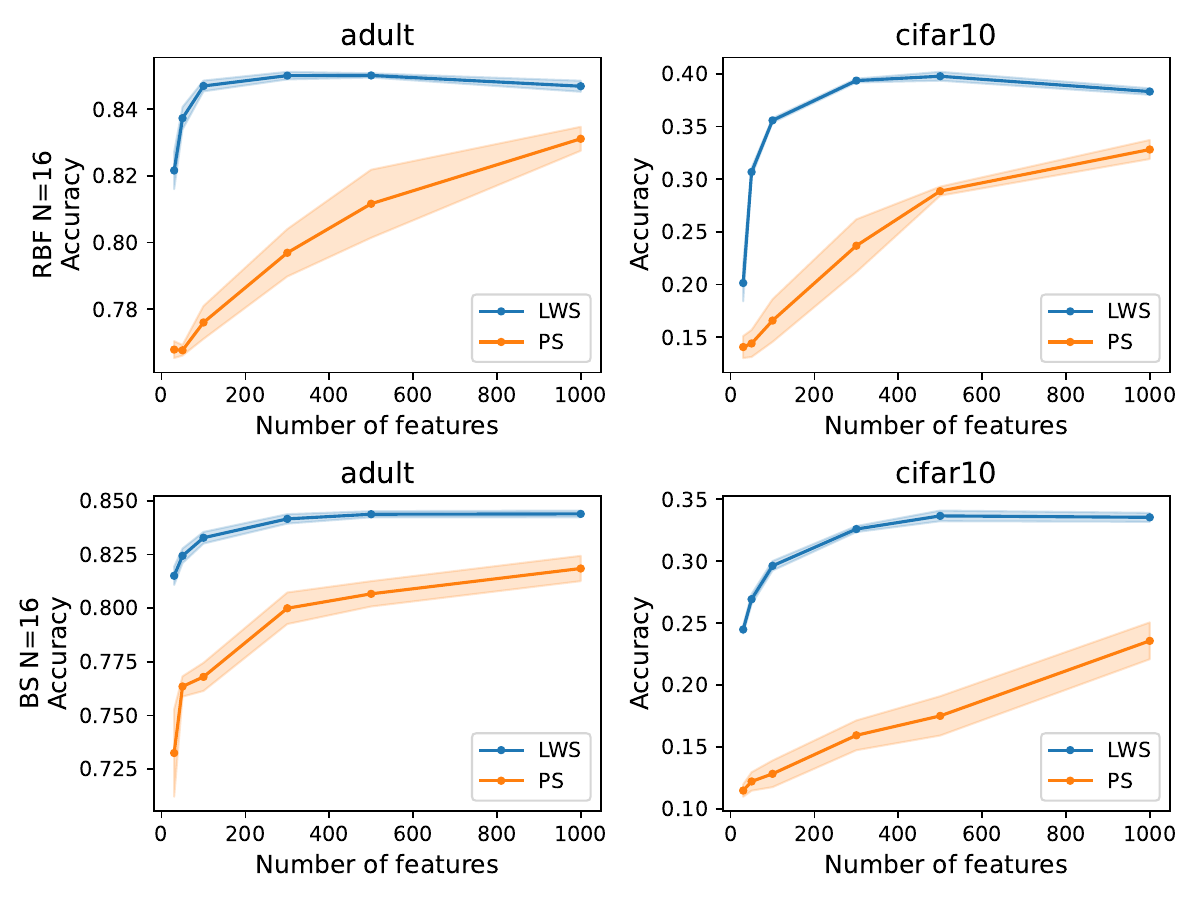}
\vspace{.3in}
\caption{\centering{Comparison of Leverage Weighted and Plain Sampling Schemes in Accuracies\\ for RFLAF of Different Basis Functions (RBF in Line 1, B-spline in Line 2)}}
\label{fig3}
\end{figure*}

\begin{figure*}[tbp]
\vspace{.3in}
\centering
\includegraphics[width=0.6\linewidth]{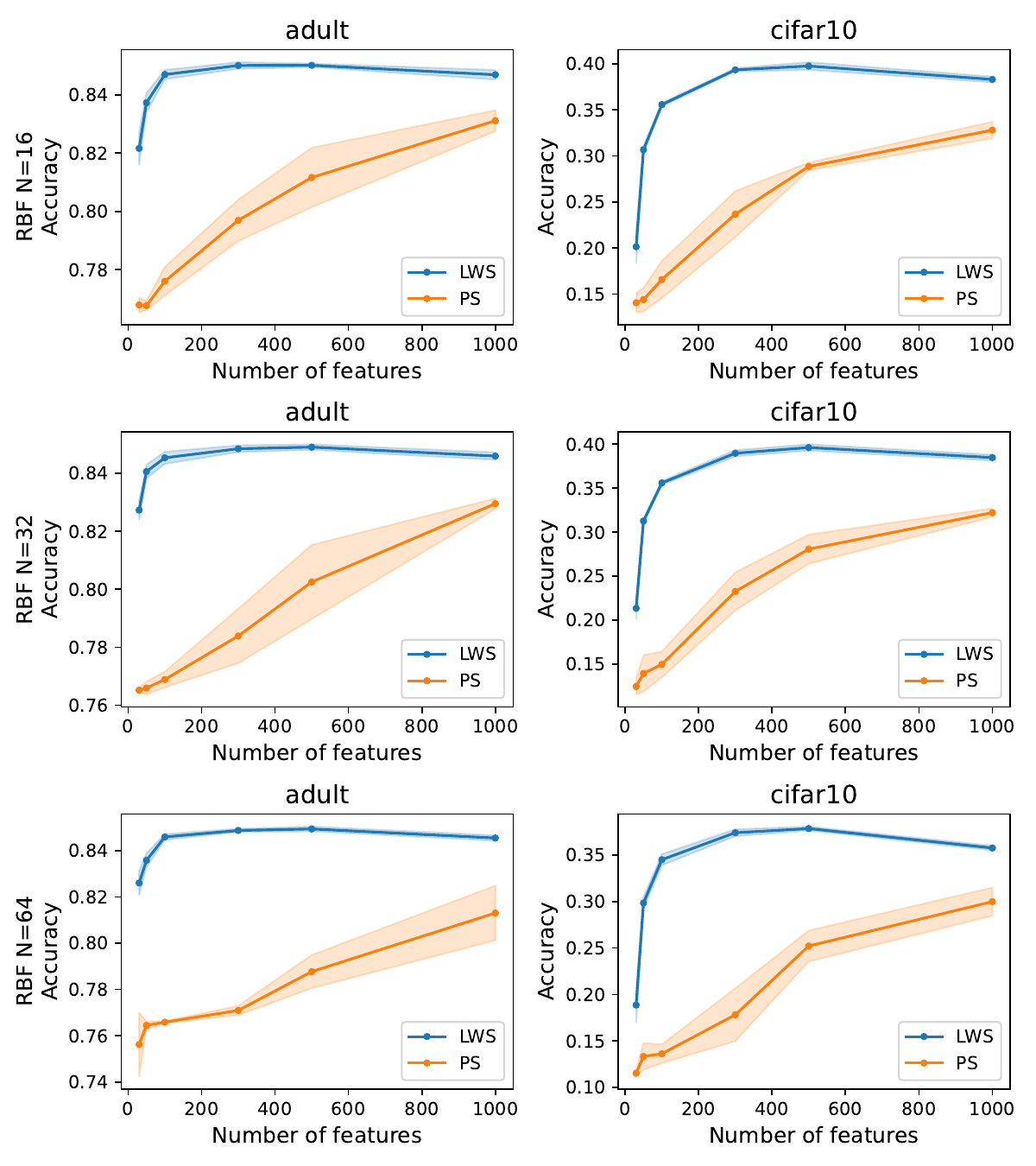}
\vspace{.3in}
\caption{\centering{Comparison of Leverage Weighted and Plain Sampling Schemes in Accuracies\\ for RFLAF of Different Grid Numbers ($N=16,32,64$ from top to bottom)}}
\label{fig4}
\end{figure*}

\end{document}